\documentclass{article}
\usepackage{fullpage,appendix}
\usepackage{subcaption}
\usepackage{cite}

\usepackage{microtype}
\usepackage{graphicx}
\usepackage{booktabs} 
\usepackage{multirow}
\usepackage[table,xcdraw]{xcolor}
\usepackage{makecell,enumitem}

\usepackage{hyperref}

\usepackage{amsmath}
\usepackage{amssymb}
\usepackage{mathtools}
\usepackage{amsthm}
\usepackage{algpseudocode}
\usepackage{algorithm}
\usepackage{graphicx}
\usepackage{booktabs}
\usepackage{multirow}
\usepackage{float}

\usepackage[capitalize,noabbrev]{cleveref}

\newtheorem{theorem}{Theorem}
\newtheorem{lemma}{Lemma}
\newtheorem{corollary}{Corollary}

\allowdisplaybreaks

\usepackage[textsize=tiny]{todonotes}

\title{Certified Adversarial Robustness via Partition-based Randomized Smoothing}

\date{}

\author{
Hossein~Goli\thanks{Department of Computer Engineering, Sharif University of Technology, \url{hossein.goli@sharif.edu}}~,
Farzan~Farnia\thanks{Department of Computer Science and Engineering, The Chinese University of Hong Kong, \url{farnia@cse.cuhk.edu.hk}}
	}     

\begin{document}
\maketitle

\begin{abstract}
A reliable application of deep neural network classifiers requires robustness certificates against adversarial perturbations. 
Gaussian smoothing is a widely analyzed approach to certifying robustness against norm-bounded perturbations, where the certified prediction radius depends on the variance of the Gaussian noise and the confidence level of the neural net's prediction under the additive Gaussian noise. However, in application to high-dimensional image datasets, the certified radius of the plain Gaussian smoothing could be relatively small, since Gaussian noise with high variances can significantly harm the visibility of an image. In this work, we propose the \emph{Pixel Partitioning-based Randomized Smoothing (PPRS)} methodology to boost the neural net's confidence score and thus the robustness radius of the certified prediction. We demonstrate that the proposed PPRS algorithm improves the visibility of the images under additive Gaussian noise. We discuss the numerical results of applying PPRS to standard computer vision datasets and neural network architectures. Our empirical findings indicate a considerable improvement in the certified accuracy and stability of the prediction model to the additive Gaussian noise in randomized smoothing.  
\end{abstract}

\section{Introduction}
\label{sec:intro}
While deep neural network (DNN) classifiers have attained state-of-the-art performance in benchmark image and sound recognition tasks, they are widely known to be vulnerable to minor perturbations to their input data, commonly regarded as \emph{adversarial attacks}. Since the introduction of adversarial perturbations in \cite{szegedy2013intriguing,biggio2013evasion,goodfellow2015explaining}, adversarial attack and defense methods have been extensively studied in the literature. In particular, the development of defense methods with certified robustness against norm-bounded perturbations has received significant attention over recent years.

\begin{figure}[ht]
\centering
\begin{tabular}{c}
\includegraphics[trim={2cm 11.55cm 2cm 11.55cm}, width=0.6\linewidth]{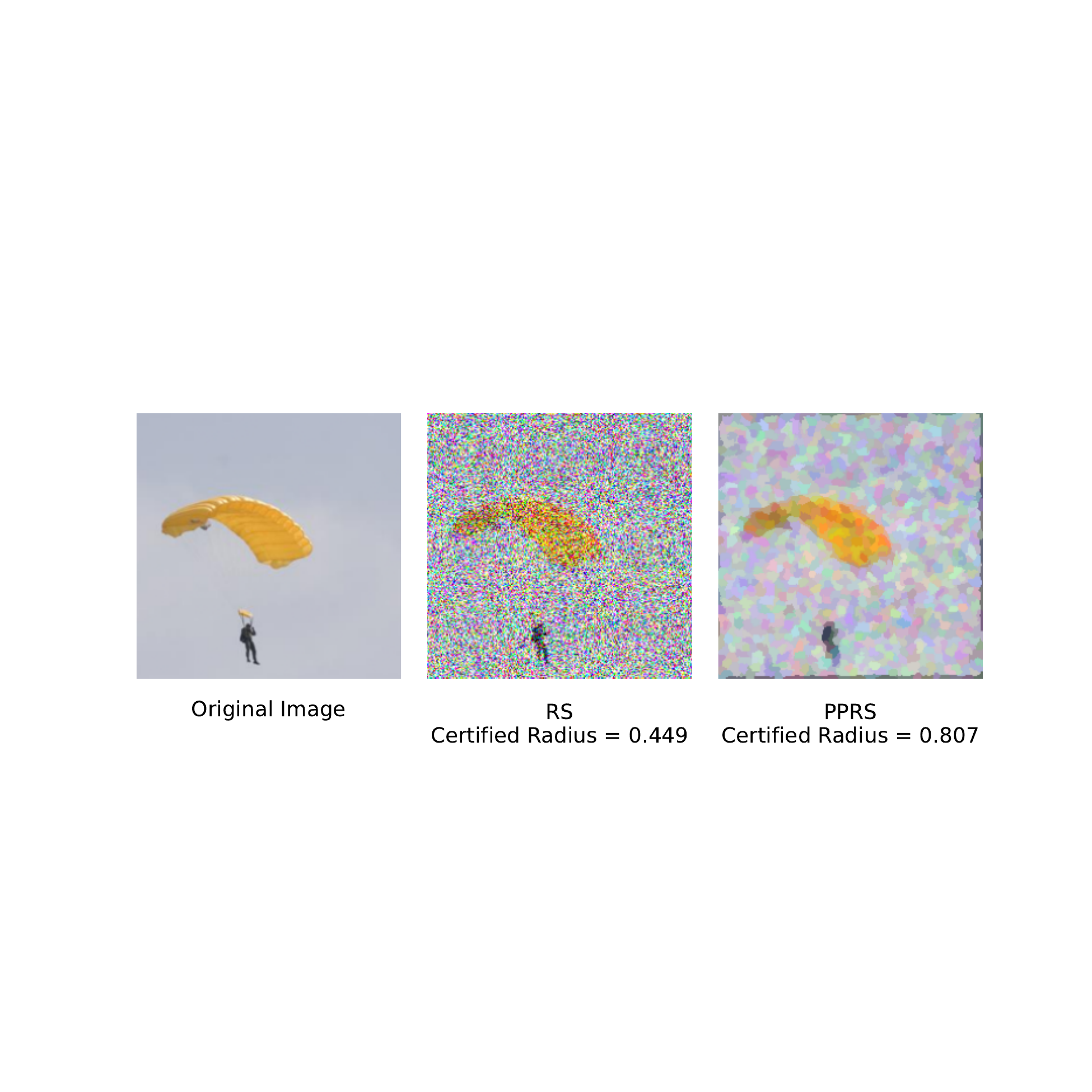}\\ [1ex]
\includegraphics[trim={2cm 11.55cm 2cm 11.55cm}, width=0.6\linewidth]{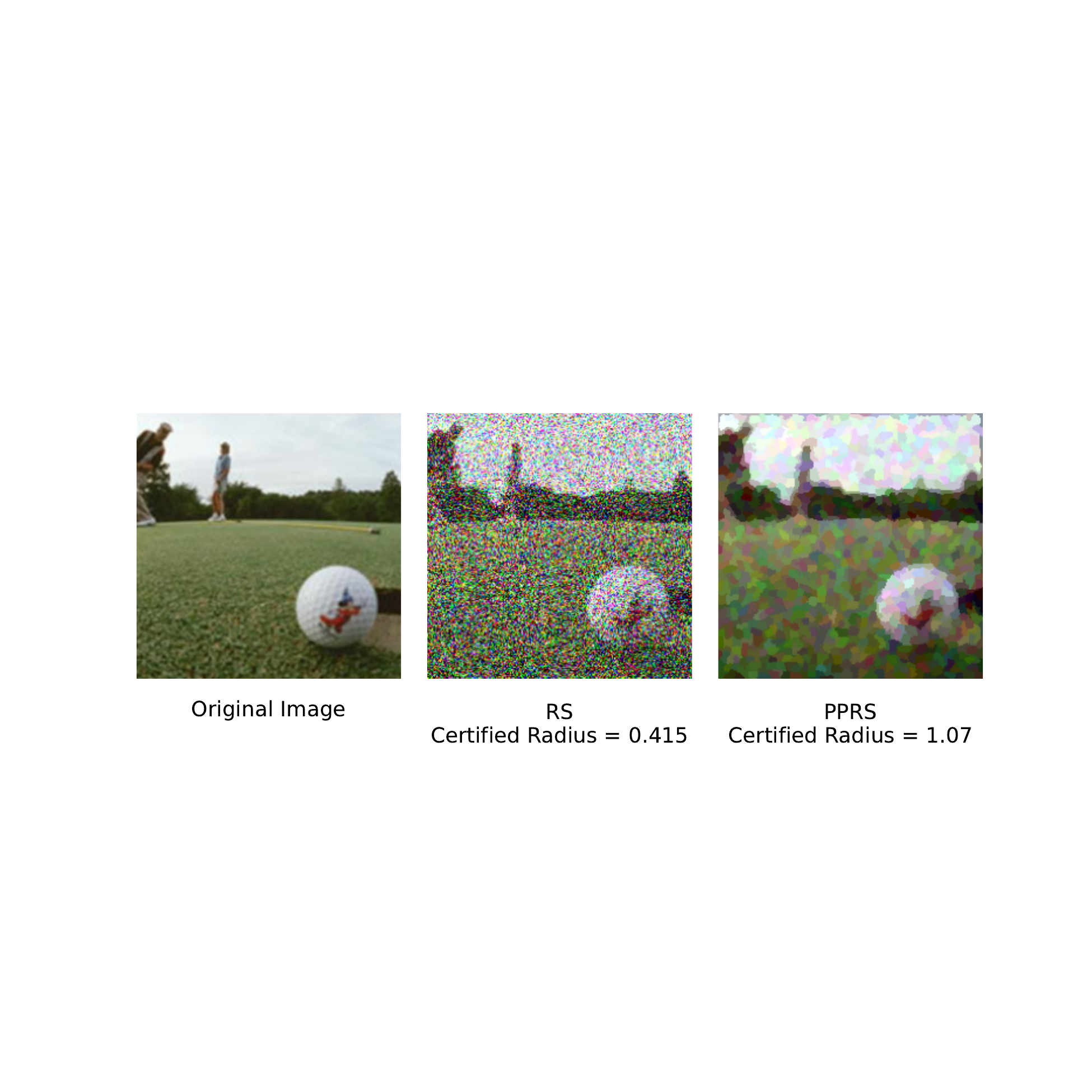}\\ [1ex]
\includegraphics[trim={2cm 11.55cm 2cm 11.55cm}, width=0.6\linewidth]{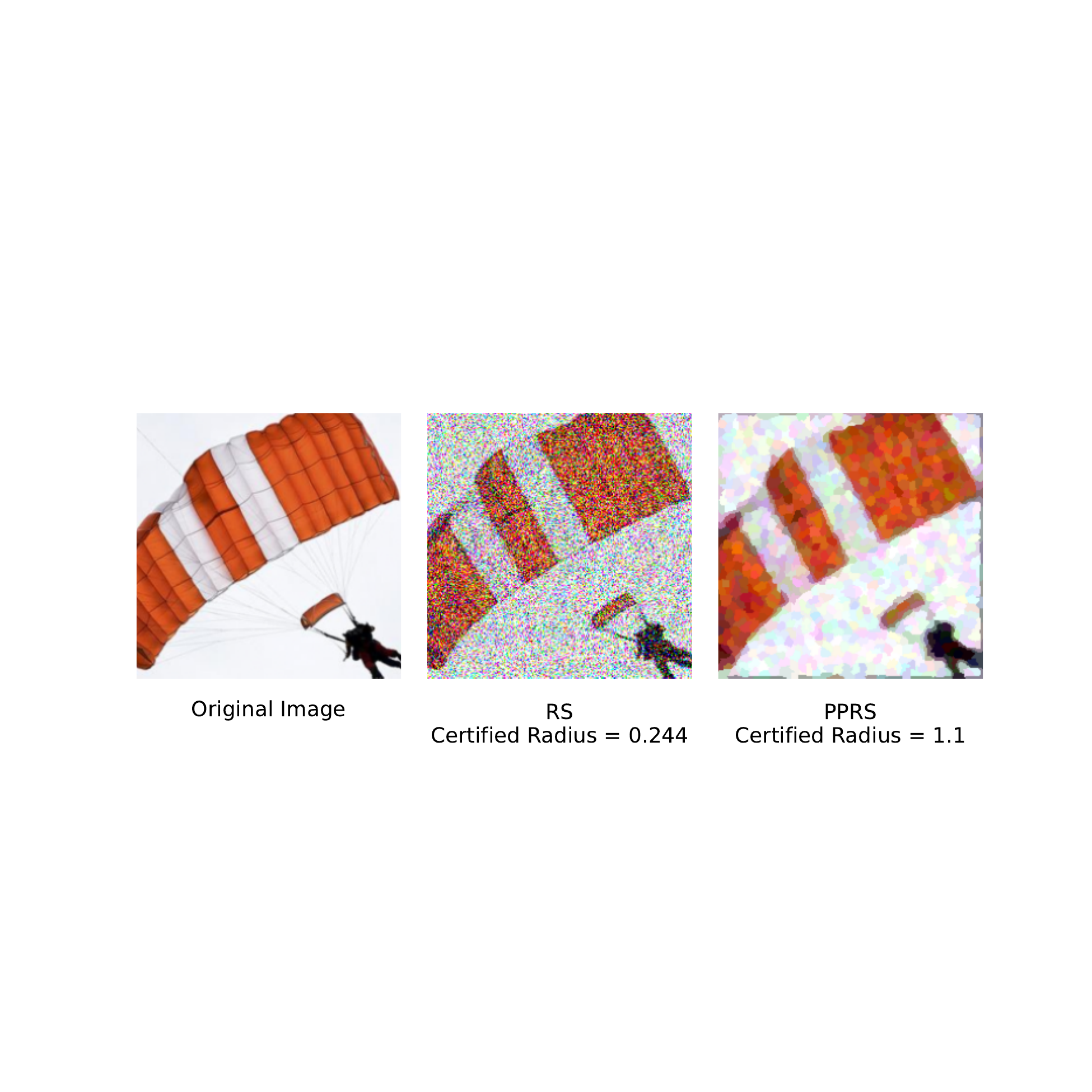}\\ [1ex]
\end{tabular}

\caption{Comparison of the randomized smoothing (RS) and the proposed PPRS under Gaussian noise with $\sigma=0.5$.}
\label{fig:cont_int}%
\end{figure}
A standard approach to certifiably robust classification is applying randomized smoothing to the DNN classifier, where the DNN's input is perturbed by a random perturbation. As demonstrated in \cite{cohen2019certified}, an additive Gaussian noise with an isotropic covariance matrix $\sigma^2 I$ can result in certified robustness against perturbations with bounded $L_2$-norm, where the certified prediction radius depends on the product of the noise standard deviation $\sigma$ and the confidence score of the DNN classifier under the additive noise. Consequently, while a higher noise variance may seem to offer a greater certified radius, the resulting drop in the prediction confidence under a higher noise variance could lead to 
 a weaker robustness certificate. 
Therefore, a potential idea for boosting the certified robustness radius of the randomized smoothing approach is to improve the classifier's prediction confidence score, which can be interpreted as the visibility of the input image, for a noisy image under a high noise variance. 

In this work, we propose the \emph{Pixel Partitioning-based Randomized Smoothing (PPRS)} method to improve the visibility of noisy images and, as a result, the certified prediction radius of deep neural network classifiers. Following the proposed PPRS method, the classifier performs a partitioning of the noisy pixels and assigns the mean pixel intensity of the partition's pixels to the pixels in every pixel partition. If the choice of the pixel partitions correlates with the image semantics, then one can expect that the partitioning-based averaging of pixel intensities in PPRS will considerably preserve the visibility of the input image. In contrast, since the additive Gaussian noise is generated independently at different pixels, the partition-based averaging can significantly reduce the effective variance of the Gaussian noise. As a result, under a semantically meaningful partitioning of the image pixels, the PPRS transformation is expected to improve the signal-to-noise ratio of the randomly perturbed image. For example, Figure~\ref{fig:cont_int} displays three ImageNet examples where the PPRS approach improves the visibility of the noisy image and thus the certified radius of the robust classification.

To perform a semantically relevant partitioning of an image’s pixels, we apply standard \textit{super-pixel} methods \cite{felzenszwalb2006efficient,vedaldi2008quick,achanta2012slic}, which are widely used in the computer vision literature. Super-pixels are commonly generated by unsupervised clustering of the pixels into subareas with similar pixel intensities. Since super-pixel algorithms do not use knowledge of other samples, they can avoid overfitting issues. Therefore, they are capable of mapping a high-dimensional image input to a super-pixel set with significantly lower dimensions. 

We show that the standard certified classification guarantee for randomized smoothing in \cite{cohen2019certified} can be extended to the PPRS classification rule under both statically and dynamically selected partitions. For a static selection of the pixel groups, the effective variance of the isotropic Gaussian noise is shown to drop by the partition size factor. Furthermore, for dynamically selected partitions, such as super-pixels, we analyze and bound the changes in the pixel clusters under the partition-based randomized smoothing in PPRS. 

Finally, we numerically evaluate the certified accuracy achieved by the standard Gaussian smoothing and the super-pixel-based PPRS algorithm on standard MNIST, CIFAR-10, and ImageNet datasets. Our empirical results suggest that PPRS can considerably improve the certified accuracy of neural network classifiers. Also, we visualize the randomly perturbed samples to show the improved visual quality of the noisy image after applying the PPRS transformation, which could result in stronger robustness certificates compared to vanilla randomized smoothing. 

The following is a summary of this work's main contributions:
\begin{itemize}
    \item Develop PPRS as a partitioning-based randomized smoothing approach to certified robust classification,
    \item Propose the application of super-pixels to improve the visibility of images under PPRS-smoothed random noise,
    \item Conduct a numerical study comparing the certified accuracy of vanilla and PPRS randomized smoothing methods.
\end{itemize}

\section{Related Work}
Certified robustness against adversarial attacks has been extensively studied in the literature. The randomized smoothing approach in \cite{cohen2019certified} provides a standard framework to certify the predictions of a general classifier against adversarial noise. This approach has been extended to adversarial perturbations with bounded $L_0$-norm \cite{levine2020robustness,levine2020randomized}, $L_1$-norm \cite{levine2021improved}, and $L_\infty$-norm \cite{zhang2021towards}. The randomized smoothing technique has also been utilized for non-classification tasks, including the image segmentation \cite{pmlr-v139-fischer21a} and community detection \cite{jia2020certified}. In addition, \cite{pmlr-v119-yang20c} proposes a general framework, extending \cite{cohen2019certified}'s method to find a robust certified radius for any noise distribution on different norms ($L_\infty$, $L_1$ and $L_2$ ).
\cite{pmlr-v168-anderson22a} proposes a new approach to certified robustness that uses training data and a direction oracle encoding information about the decision boundary, which is only partially based on randomized smoothing to certify robustness.

Regarding the extensions of the randomized smoothing approach, \cite{NEURIPS2020_300891a6} provides a framework extending \cite{cohen2019certified}'s method, but as also discussed in the paper it may not be sufficiently effective against $L_2$-norm-bounded perturbations. 
\cite{ijcai2022p467} provides a novel approach to randomized smoothing in settings with multiplicative parameters of input transformations where a gamma correction perturbation is utilized.
Concerning the applications of certified robustness, \cite{ijcai2023p767} provides an adversarial framework with certified robustness guarantees for time series data using its statistical features.

On the other hand, the effectiveness of randomized smoothing has been evaluated and examined in \cite{maho2022randomized}. This paper reports the gap between certified accuracy of the randomized smoothing methods and the robustness achieved by standard adversarial training methods. \cite{maho2022randomized}'s findings are consistent with our observation of the dropped confidence score of the classifier for noisy data, and the proposed PPRS attempts to close \cite{maho2022randomized}'s observed gap. 
Also, \cite{pmlr-v119-kumar20b} discusses the curse of dimensionality of randomized smoothing, which can be mitigated by applying our proposed partitioning scheme.  This property of compact feature representations and their effects on 
adversarial robustness and generalizability have also been discussed in \cite{9414696}.

\section{Preliminaries}
\subsection{Classification and Adversarial Attacks}
Consider a labeled sample $(\mathbf{x},y)$ where $\mathbf{x}\in\mathcal{X}\subseteq\mathbb{R}^d$ denotes a $d$-dimensional feature vector and $y\in\mathcal{Y}=\{1,2,\ldots, k\}$ is a $k$-ary label. A prediction function $f:\mathcal{X}\rightarrow \mathcal{Y}$ maps the input $\mathbf{x}$ to a label in $\mathcal{Y}$. We use the standard 0/1-loss to evaluate the performance of the classifier on a labeled sample, i.e. the loss will be zero if the prediction is correct and, otherwise, will be one:
\begin{equation*}
    \ell_{0/1}\bigl(f(\mathbf{x}),y\bigr) = \begin{cases}
    0\quad &\text{\rm if}\: f(\mathbf{x})=y,\\
    1\quad &\text{\rm if}\: f(\mathbf{x})\neq y.
    \end{cases}
\end{equation*}

However, standard neural network classifiers have been frequently observed to lack robustness against norm-bounded adversarial perturbations \cite{szegedy2013intriguing,biggio2013evasion,goodfellow2015explaining}. To generate an adversarial perturbation $\boldsymbol{\delta}\in\mathbb{R}^d$ with $\epsilon$-bounded norm, one can solve the following optimization problem:
\begin{equation}\label{Eq: Adv Attack Generation}
    \max_{\boldsymbol{\delta}:\: \Vert\delta \Vert\le \epsilon}\; \ell\bigl(f\bigl(\mathbf{x}+\boldsymbol{\delta}\bigr),y\bigr), 
\end{equation}
where $\ell$ is commonly chosen as a smoothed version of the $0/1$ loss function, e.g. the cross entropy loss. Also, in this work, we focus on standard $L_2$-norm and $L_{\infty}$-norm-based adversarial perturbations in our theoretical and numerical analysis, which means we choose the norm $\Vert\cdot\Vert$ in \eqref{Eq: Adv Attack Generation} to be one of these norm functions.

\subsection{Certified Robustness via Randomized Smoothing}
Certifying a classifier's robustness against norm-bounded perturbations is required for high-stake machine learning applications. The randomized smoothing approach in \cite{cohen2019certified} is a standard framework to achieve certified adversarial robustness. 
Following this approach, we consider an isotropic Gaussian vector $\mathbf{Z}\sim\mathcal{N}(\mathbf{0},\sigma^2I_{d\times d})$, where $I_{d\times d}$ denotes the $d$-dimensional identity matrix, and define the following prediction rule $f^{\mathrm{GS}(\sigma)}$ for a given prediction rule $f$:
\begin{equation*}
   f^{\mathrm{GS}(\sigma)}(\mathbf{x}) := \underset{c\in\mathcal{Y}}{\arg\!\max} \:\mathbb{P}\bigl(f(\mathbf{x+Z}) = c\bigr).
\end{equation*}
In other words, $ f^{\mathrm{GS}(\sigma)}$ outputs the most likely label for a noisy version of $\mathbf{x}$ perturbed with the zero-mean Gaussian vector $\mathbf{Z}$. \cite{cohen2019certified} proves the following robustness guarantee for the defined $ f^{\mathrm{GS}(\sigma)}$:
\begin{theorem}[Theorem 1 from \cite{cohen2019certified}]\label{Thm: Thm 1 Cohen}
Consider a prediction rule $f:\mathcal{X}\rightarrow\mathcal{Y}$. For sample $\mathbf{x}\in\mathcal{X}$ classified as $f^{\mathrm{GS}(\sigma)}(\mathbf{x}) = c_A$, we define the prediction confidence score as $C(\mathbf{x})=\frac{1}{2}\bigl(\Phi^{-1}(p_A(\mathbf{x})) - \Phi^{-1}(p_B(\mathbf{x}))\bigr)$ where $\Phi^{-1}:(0,1)\rightarrow\mathbb{R}$ is the inverse-CDF of the standard Gaussian distribution $\mathcal{N}(0,1)$ and $p_A(\mathbf{x}),\, p_B(\mathbf{x})$ are defined as:
\begin{align*}
    p_A(\mathbf{x}) &= \mathbb{P}\bigl(f(\mathbf{x+Z}) = c_A \bigr), \\
    p_B(\mathbf{x}) &= \max_{c\neq c_A}\:\mathbb{P}\bigl(f(\mathbf{x+Z}) = c \bigr).
\end{align*}
Then, $f^{\mathrm{GS}(\sigma)}(\mathbf{x}+\boldsymbol{\delta})=c_A$ for every $L_2$-norm-bounded perturbation $\Vert \boldsymbol{\delta}\Vert_2\le \sigma C(\mathbf{x}).$
\end{theorem}
In the next sections, we will discuss the trade-off between the standard deviation parameter $\sigma$ for the additive Gaussian noise and the averaged prediction confidence score $c(\mathbf{x})$ under the random Gaussian perturbation.


\section{Partition-based Randomized Smoothing}
As discussed earlier, the certified prediction robustness in the randomized smoothing approach is the product of the standard deviation $\sigma$ of the Gaussian noise and the prediction confidence score under the additive noise. However, if $\sigma$ is chosen to be moderately large, the average prediction confidence score could drop significantly, leading to a smaller certified prediction radius and hence lower certified accuracy. Therefore, one idea to improve the certified performance by randomized smoothing is to develop a method to improve the visibility of images affected by additive Gaussian noise.

\begin{figure}[t]
    \centering
    \begin{subfigure}{\textwidth}
    \includegraphics[trim={5cm 22.3cm 5cm 20.3cm},clip, scale=0.4]
    {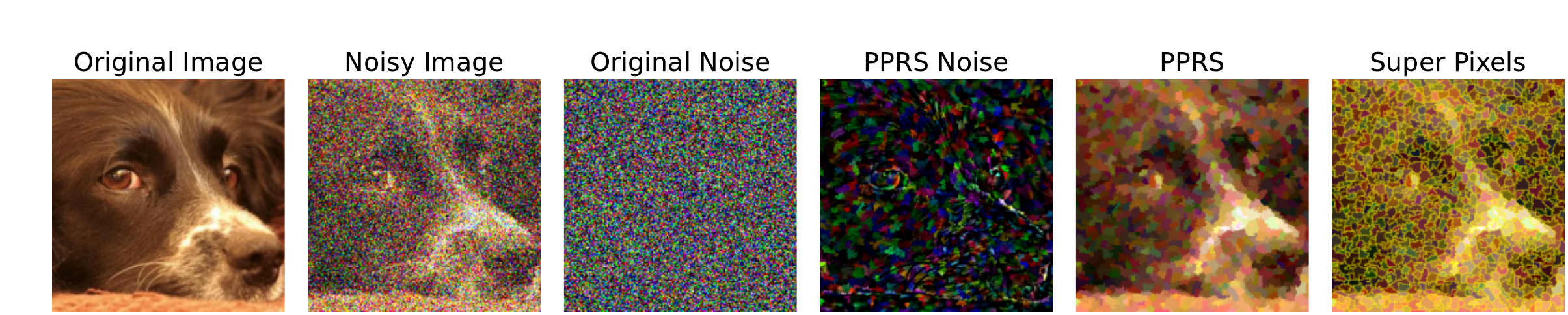}
    \end{subfigure} 
    \begin{subfigure}{\textwidth}
    \includegraphics[trim={5cm 22.3cm 5cm 22.3cm},clip, scale=0.4]
    {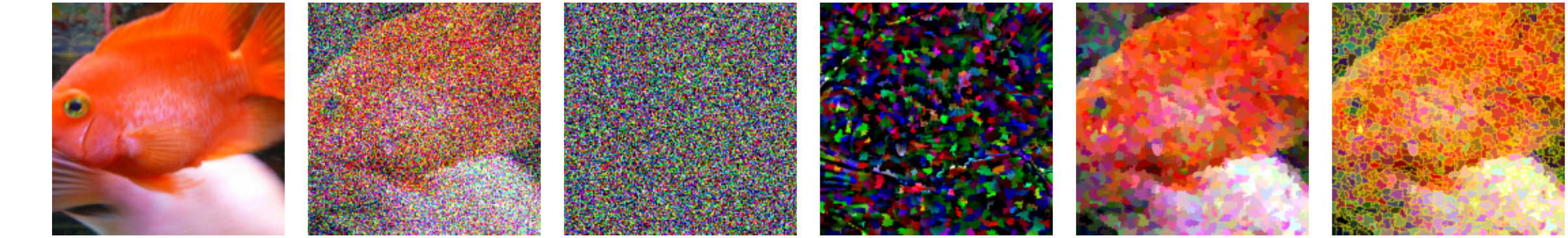}
    \end{subfigure} 
    \begin{subfigure}{\textwidth}
    \includegraphics[trim={5cm 22.3cm 5cm 22.3cm},clip, scale=0.4]{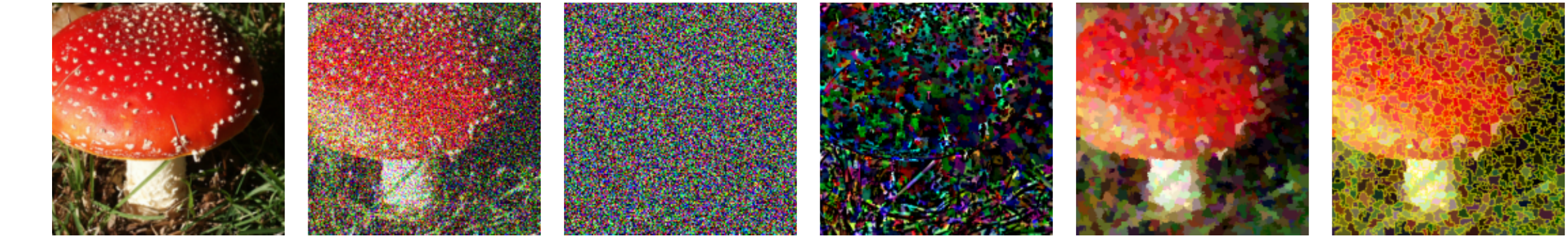}
    \end{subfigure} 
    \begin{subfigure}{\textwidth}
    \includegraphics[trim={5cm 21.3cm 5cm 22.3cm},clip, scale=0.4]{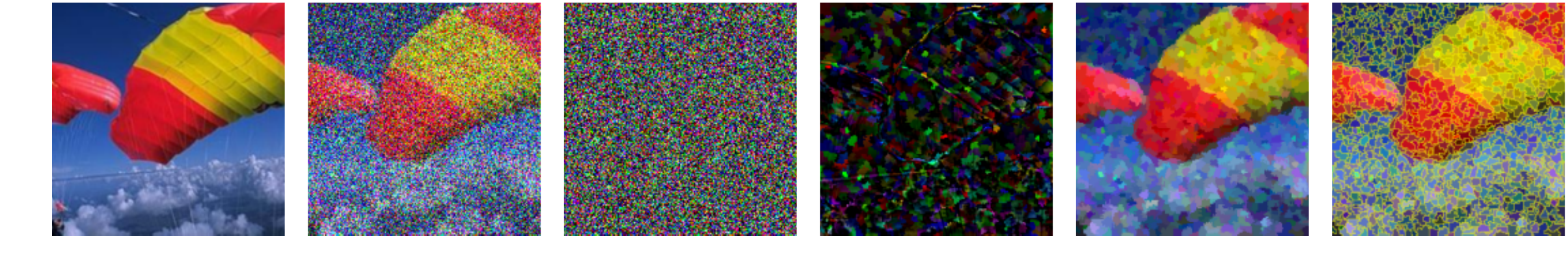}
    \end{subfigure}\hfil 
    \caption{  The two top and bottom rows demonstrate pictures found after adding Gaussian Noise with $\sigma=0.75$ and $\sigma=0.5$ and applying SLIC SuperPixel with 1000 partitions. }
    \label{fig:imgs}

\end{figure}

\begin{figure*}[h]
    \centering
    \begin{subfigure}{0.485\textwidth}
  \includegraphics[trim={0 0 0 0},clip ,width=\linewidth]{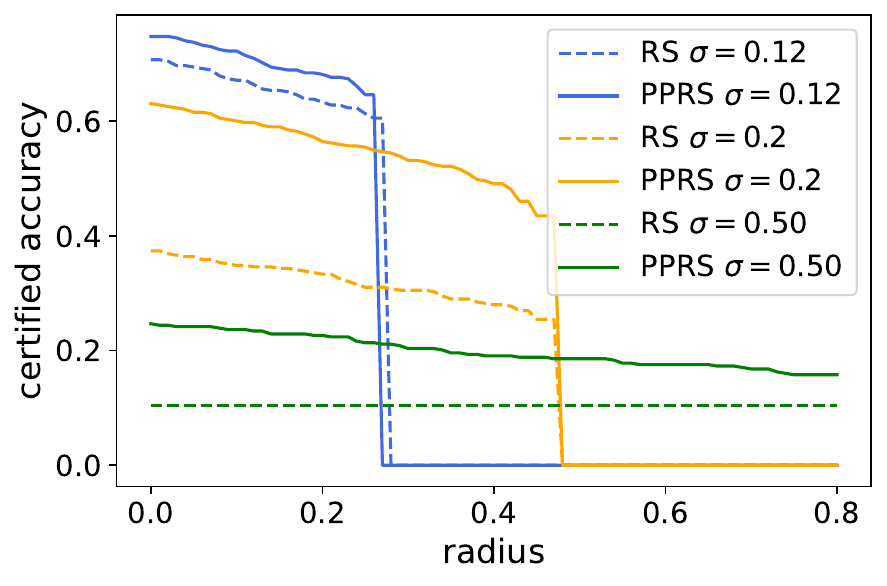}
  \caption{ImageNet results}
  \label{fig:MainExp ImageNetSub}
\end{subfigure}\hfil 
\begin{subfigure}{0.485\textwidth}
  \includegraphics[trim={0 0 0 0},clip ,width=\linewidth]{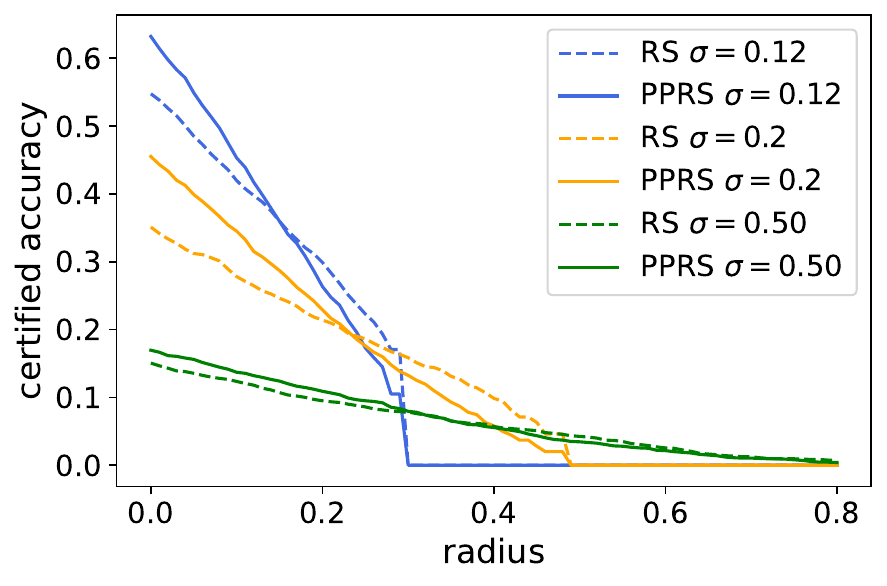}
  \caption{CIFAR-10 results}
  \label{fig:MainExp Cifar}
\end{subfigure}\hfil 
  \caption{PPRS and RS Certified Accuracy vs. $L_2$-Perturbation-Radius for Gaussian noise $\mathcal{N}(\mathbf{0},\sigma^2 I)$}
  \label{fig:MainExp ImageNet}
\end{figure*}

\begin{figure*}[h]
    \centering 
\begin{subfigure}{0.485\textwidth}
  \includegraphics[trim={0 0 0 0},clip ,width=\linewidth]{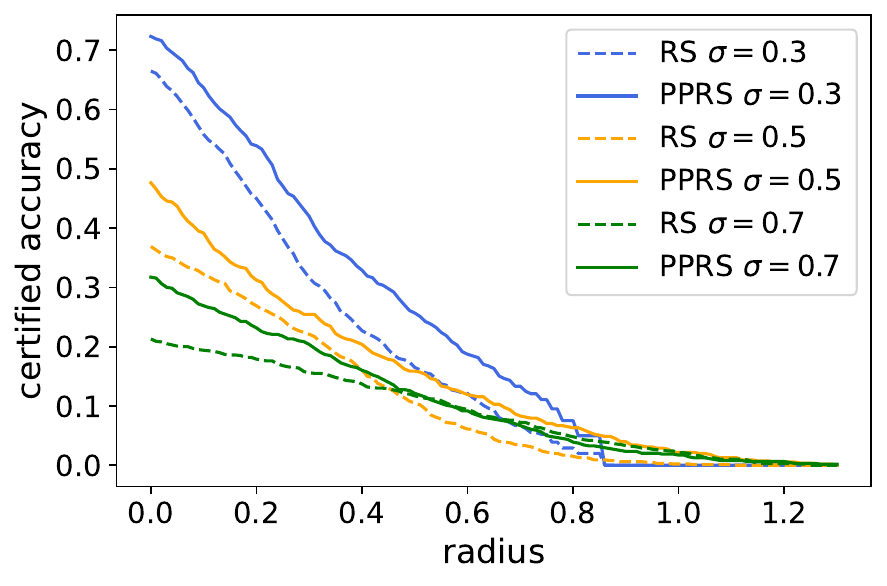}
  \caption{Fashion MNIST results}
  \label{fig:MainExp fashionMnist}
\end{subfigure}\hfil 
\begin{subfigure}{0.485\textwidth}
  \includegraphics[trim={0 0 0 0},clip ,width=\linewidth]{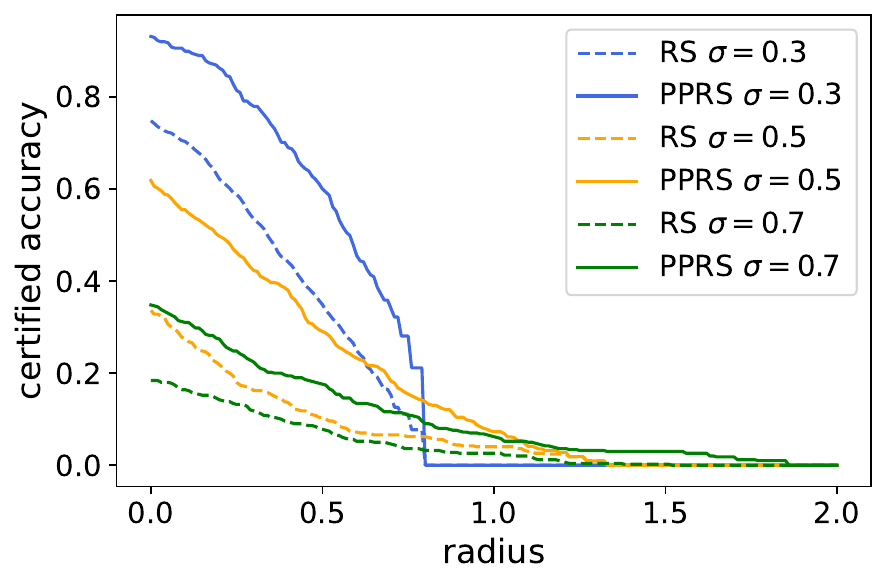}
  \caption{MNIST results}
  \label{fig:MainExp Mnist}
\end{subfigure}\hfil 
\caption{PPRS and RS Certified Accuracy vs. $L_2$-Perturbation-Radius for Gaussian noise $\mathcal{N}(\mathbf{0},\sigma^2 I)$ with different $\sigma$'s.}
\label{fig:plots}
\end{figure*}

To achieve this goal, we propose the \emph{Pixel Partitioning-based Randomized Smoothing (PPRS)}, according to which we partition the pixels into multiple groups and use the averaged pixel intensities within each partition. Mathematically, we group the pixels into $p$ partitions in $\mathbb{S}(\mathbf{x}):=\{S_1(\mathbf{x}), \cdots, S_p(\mathbf{x})\}$ where $S_i(\mathbf{x})$ denotes the subset of pixels corresponding to the $i$th superpixel. We define the partition averaging matrix $\mathbf{A}_{\mathbb{S}} \in [0, 1]^{d\times p}$ as
\begin{equation}
  \mathbf{A}_{i,j} \, :=\, \begin{cases}
  \frac{1}{|S_j|}\quad &\text{\rm if Pixel $i$}\in S_j \\
  0\quad &\text{\rm otherwise.}
  \end{cases}  
\end{equation} 
Note that the PPRS transformation of an image $\mathbf{x}$ will be $A_\mathbb{S}\mathbf{x}$, resulting in the following definition of the PPRS-based classification rule.
\begin{equation}
\mathrm{PPRS}(f,\mathbf{x}) := f\bigl(\mathbf{A}_{\mathbb{S}}\mathbf{x}\bigr).
\label{equal:group}
\end{equation}
Applying the PPRS transformation to noisy image $\mathbf{x}+\mathbf{Z}$, where $\mathbf{Z}\sim\mathcal{N}(\mathbf{0},\sigma^2I)$, results in an effective noise standard deviation $\frac{\sigma}{\sqrt{|S_j|}}$ for the pixels in partition $S_j$ because the Gaussian noise for $S_j$'s pixels are independent. Therefore, under a static selection of the partitions, when $\mathcal{S}(\mathbf{x})$ does not depend on $\mathbf{x}$, we have the following corollary of Theorem \ref{Thm: Thm 1 Cohen}.
\begin{corollary}\label{Cor: Cor 1 Static}
Suppose that the partitioning function $\mathcal{S}(\mathbf{x})$ assigns the same partition to every $\mathbf{x}$. Then, according to the definitions of Theorem \ref{Thm: Thm 1 Cohen}, $\mathrm{PPRS}(f^{\mathrm{GS}(\sigma)},\mathbf{x}+\boldsymbol{\delta})$ outputs the same prediction for every $\boldsymbol{\delta}$ satisfying $$\bigl\Vert A_S\boldsymbol{\delta}\bigr\Vert_2 \,\le\, \sigma C_{\mathrm{PPRS}(f^{\mathrm{GS}(\sigma)})}(\mathbf{x})$$ 
where $C_{\mathrm{PPRS}(f^{\mathrm{GS}(\sigma)})}$ denotes the prediction confidence score for $\mathrm{PPRS}(f^{\mathrm{GS}(\sigma)})$. 
\end{corollary}
Therefore, Corollary \ref{Cor: Cor 1 Static} implies that the effective standard deviation of the averaged pixel intensity in each partition $S_j$ will be $\frac{\sigma}{\sqrt{|S_j|}}$, which is smaller than $\sigma$ by a factor ${\sqrt{|S_j|}}$. Therefore, if the partitions chosen in $\mathbb{S}$ preserve the information of the image $\mathbf{X}$ with respect to the label $Y$, we expect a significant increase in the value of $C_{\mathrm{PPRS}(f^{\mathrm{GS}(\sigma)})}$ compared to $C_{f^{\mathrm{GS}(\sigma)}}$ due to the smaller effective noise variance. In practice, the proper choice of pixel partitions is expected to be image-dependent. As a result, we show the following extension of Corollary \ref{Cor: Cor 1 Static} to dynamically selected partitions of a general $S(\mathbf{x})$.
\begin{theorem}\label{Cor: Cor 1 Dynamic}
Consider a dynamic partitioning scheme $\mathcal{S}(\mathbf{x})$ that for $L_2$-operator norm $\Vert\cdot \Vert_2$ and every $\mathbf{x},\mathbf{x}'\in\mathcal{X}$ satisfies $\Vert A_{\mathcal{S}(\mathbf{x})}-  A_{\mathcal{S}(\mathbf{x}')}\Vert_2 \le \rho \Vert \mathbf{x}-\mathbf{x}'\Vert_2$. Then, given the definitions of Theorem \ref{Thm: Thm 1 Cohen}, $\mathrm{PPRS}(f^{\mathrm{GS}(\sigma)},\mathbf{x}+\boldsymbol{\delta})$ results in the same prediction for every $\boldsymbol{\delta}$ satisfying $$\bigl\Vert A_{\mathcal{S}(\mathbf{x})}\boldsymbol{\delta}\bigr\Vert_2 \, \le\, (1-\rho)\sigma C_{\mathrm{PPRS}(f^{\mathrm{GS}(\sigma)})}(\mathbf{x})$$ where $C_{\mathrm{PPRS}(f^{\mathrm{GS}(\sigma)})}$ denotes the prediction confidence score for $\mathrm{PPRS}(f^{\mathrm{GS}(\sigma)})$. 
\end{theorem}
\begin{proof}
We defer the proof to the Appendix.
\end{proof}
Therefore, if the output of the partitioning scheme changes with a bounded rate across different inputs, we can still expect a similar gain in the certified accuracy of the PPRS approach. 

To dynamically select the partitions, we propose applying standard super-pixels which refer to the perceptual partitioning of pixels. The grouping in super-pixels can be performed by the clustering of pixels according to pixel characteristics such as intensity and color. In the literature, several established algorithms have been developed to assign the super-pixels such as Felzenszwalb's method~\cite{felzenszwalb2006efficient}, Quickshift~\cite{vedaldi2008quick}, and SLIC~\cite{achanta2012slic}. In practice, pixels within a super-pixel are typically relevant and belong to the same object. Hence, it is a reasonable assumption that the pixels in a super-pixel have a similar impact on the prediction rule's accuracy.    

\section{Numerical Results}
We numerically evaluated our proposed PPRS methodology on three standard image datasets: MNIST \cite{lecun-mnisthandwrittendigit-2010}, Fashion~MNIST \cite{xiao2017fashion},  CIFAR-10 \cite{cifar}, and ImageNet \cite{deng2009imagenet}.
In the experiments, we used standard pre-trained ResNet-18 models \cite{he2016deep} on the three datasets as the classification model. 
Also, for the ImageNet experiments, we considered a subset of ten labels to mitigate the computational and statistical costs of estimating the classifier's label probabilities. 

We compared the results of the proposed PPRS method with those of the standard randomized smoothing (RS) method in \cite{cohen2019certified}. As presented in Theorem \ref{Thm: Thm 1 Cohen}, we calculate the certification radius using 
    $R = \sigma C(\mathbf{x})$,
where $\sigma$ is the standard deviation of the Gaussian noise added to the image and $C(\mathbf{x})$ is the confidence score of the PPRS or RS approach. In the CIFAR-10 and ImageNet experiments, we respectively considered 250 and 1000 partitions per image assigned by the SLIC super-pixel \cite{achanta2012slic}, where on average every super-pixel contains 10.7 and 50 pixels. We defer the details of the numerical setting to the Appendix. 


\subsection{Numerical Comparison of PPRS vs. Baseline Randomized Smoothing Methods}
Figure~\ref{fig:imgs} visualizes the original and noisy versions of four randomly selected ImageNet samples in the experiments.  These examples suggest that the PPRS method could relatively increase the visual quality of the samples and, as a result, the confidence score of the neural net classifier. 
In addition, in the Appendix, we provide additional numerical results supporting that the proposed superpixel-based algorithm will be more effective under higher noise variances, since grouping of pixels helps to reduce the effective noise variance and improve input image's visibility, as shown in Figure~\ref{fig:imgs}. 

In addition, Figure \ref{fig:MainExp ImageNet} shows ImageNet-based and CIFAR-10-based comparisons between certified accuracy scores achieved by the vanilla randomized smoothing (RS) and PPRS methods for different standard deviation parameters of the additive Gaussian noise. These plots indicate that using the same noise standard deviation, the super-pixel-based randomized smoothing can achieve a higher certified accuracy at the same robustness radius, which could be attributed to the higher visibility of perturbed samples after the partition-based averaging in PPRS. 

\begin{figure}[h]
    \centering 
\begin{subfigure}{0.485\textwidth}
  \includegraphics[trim={0 0 0 0},clip ,width=\linewidth]{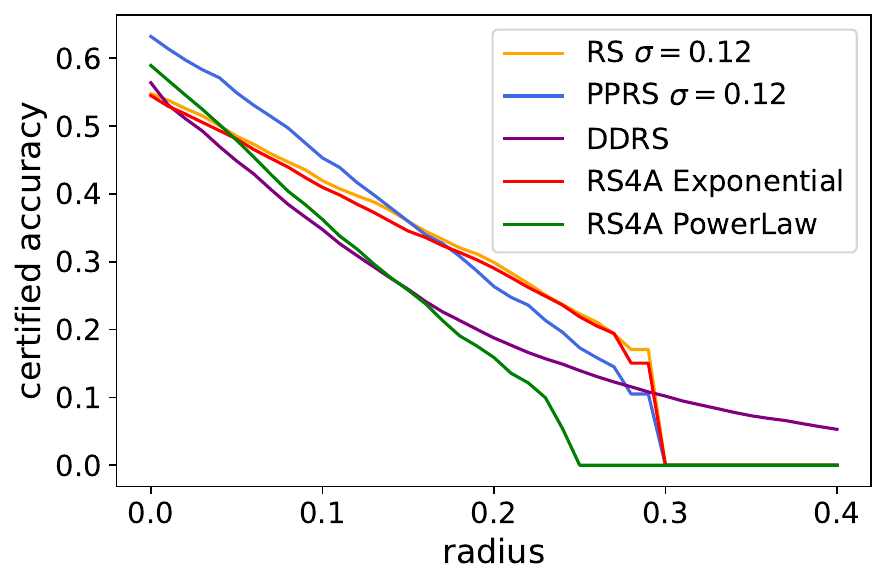}
  \caption{CIFAR-10 results}
  \label{fig:AllMethods Cifar}
\end{subfigure}\hfil 
\begin{subfigure}{0.485\textwidth}
  \includegraphics[trim={0 0 0 0},clip ,width=\linewidth]{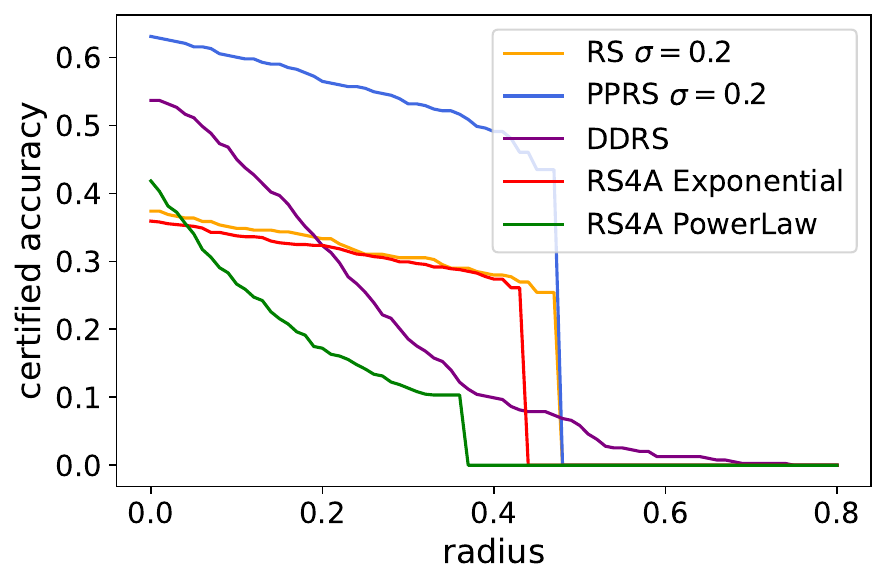}
  \caption{ImageNet results}
  \label{fig:AllMethods ImgNet}
\end{subfigure}\hfil 
\begin{subfigure}{0.485\textwidth}
  \includegraphics[trim={0 0 0 0},clip ,width=\linewidth]{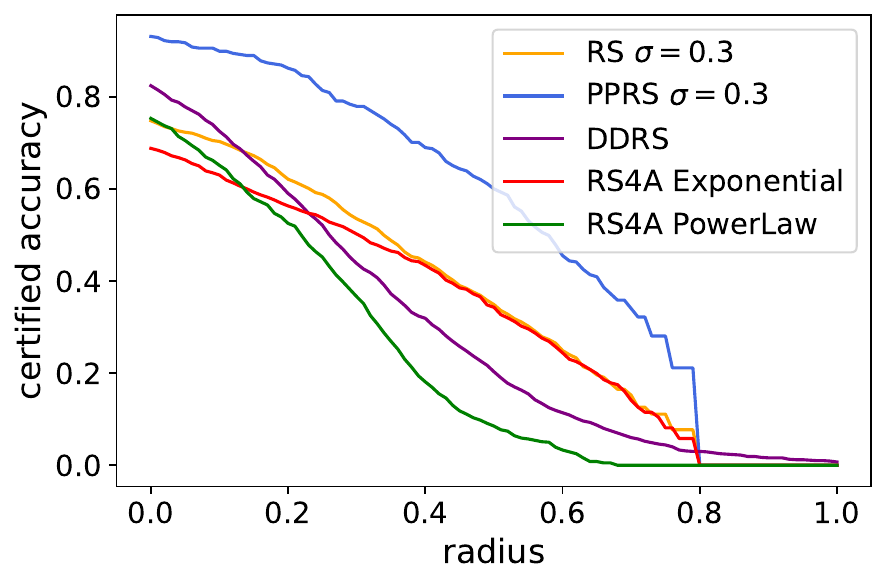}
  \caption{MNIST results}
  \label{fig:AllMethods Mnist}
\end{subfigure}\hfil 
\begin{subfigure}{0.485\textwidth}
  \includegraphics[trim={0 0 0 0},clip ,width=\linewidth]{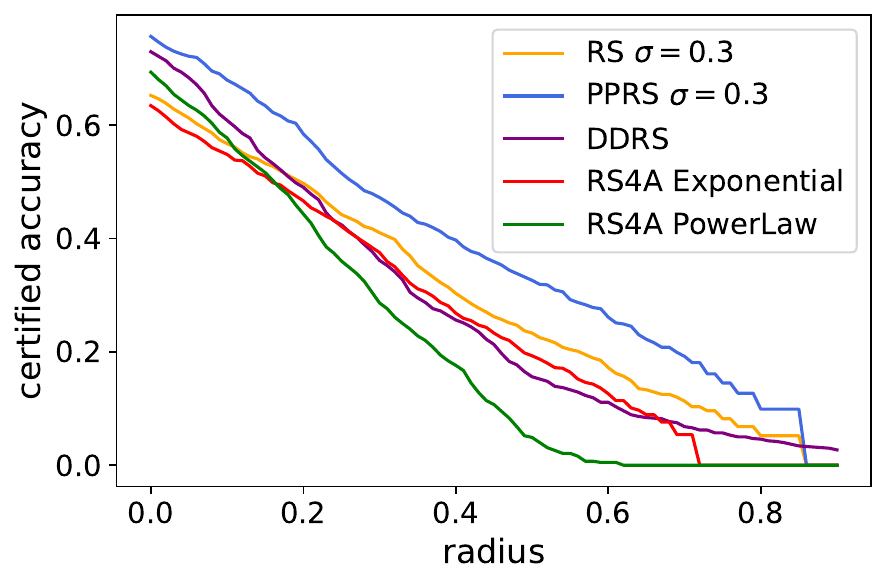}
  \caption{Fashion-MNIST results}
  \label{fig:AllMethods Mnist}
\end{subfigure}\hfil 

\caption{Numerical comparison of certifiably robust classification methods. The certified accuracy of the trained models is plotted vs. the certified prediction radius.}
\label{fig:allMethods}
\end{figure}

\subsection{Ablation Studies}

\paragraph{Super-pixel Schemes.}

To study the effect of choosing different super-pixel schemes, as shown in  Figure~\ref{fig: Ablation schemes pic}, we computed and visualized the super-pixels obtained by three different super-pixel schemes: Felzenszwalb’s method \cite{felzenszwalb2006efficient}, Quickshift \cite{vedaldi2008quick}, and SLIC \cite{achanta2012slic}. Note that the Quickshift method is a local mode-seeking algorithm that functions based on the color and location of the pixels. On the other hand, Felzenszwalb’s method is a graph-based segmentation algorithm. The empirical results in Figure~\ref{fig: Ablation schemes pic} indicate that the super-pixel-based partitioning followed by PPRS is flexible to the choice of the super-pixel scheme and can perform satisfactorily using different schemes. All the mentioned methods performed effectively compared to vanilla randomized smoothing, validating our hypothesis on the impacts of partitioning pixels on the robustness offered by randomized smoothing.

\begin{figure}
    \centering
    \begin{subfigure}{\textwidth}
    \includegraphics[trim={4cm 9cm 4cm 8cm},clip, scale=0.35]{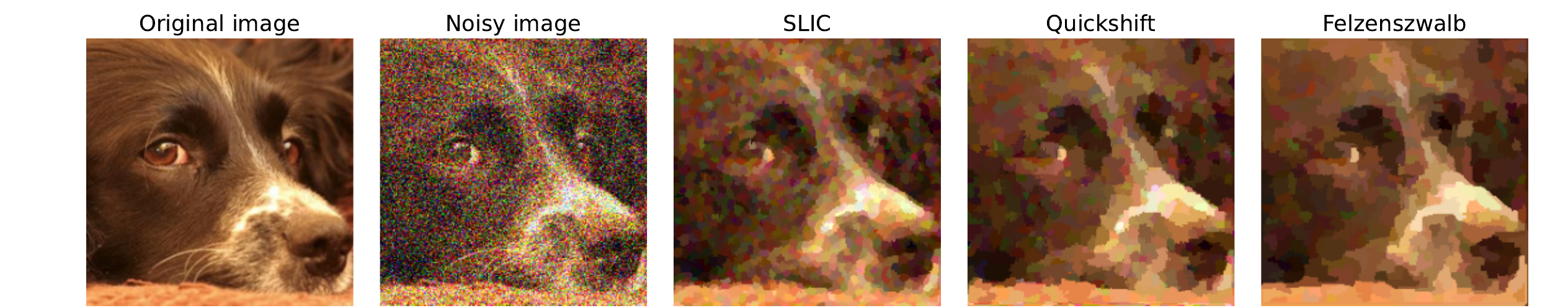}
    \end{subfigure} 
    \begin{subfigure}{\textwidth}
    \includegraphics[trim={4cm 9cm 4cm 9cm},clip, scale=0.35]{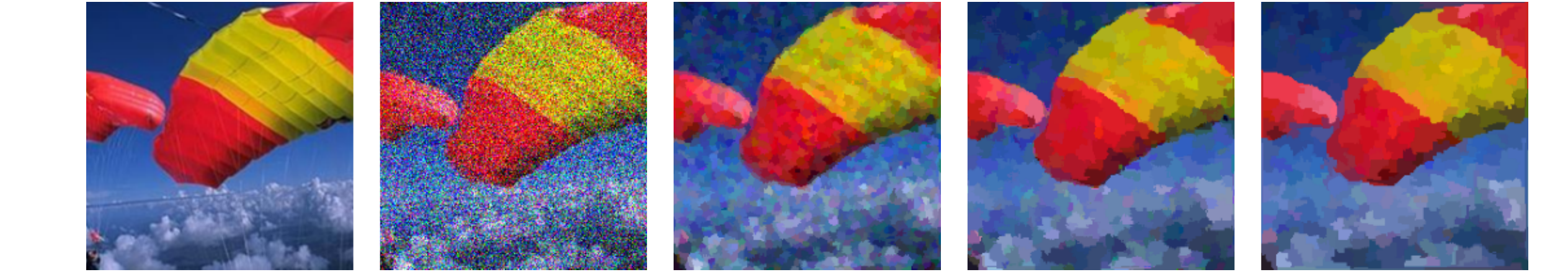}
    \end{subfigure} 
    \begin{subfigure}{\textwidth}
        \includegraphics[trim={4cm 9cm 4cm 9cm},clip, scale=0.35]{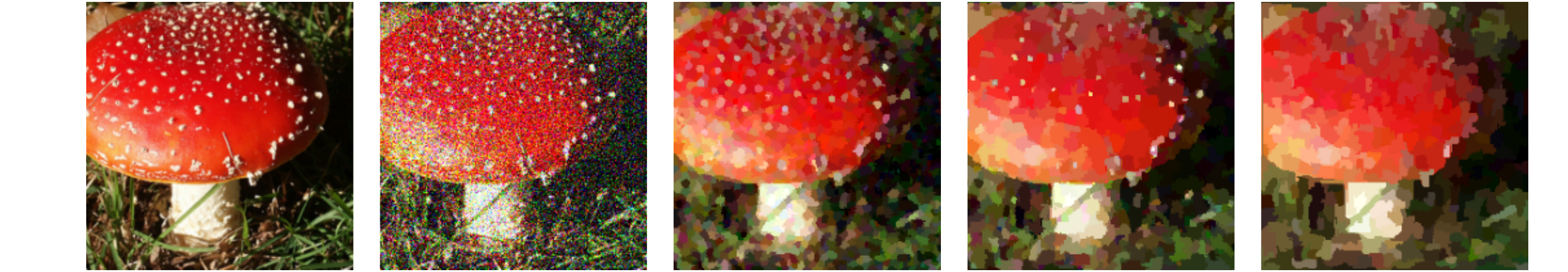}
    \end{subfigure} 
    \begin{subfigure}{\textwidth}
        \includegraphics[trim={4cm 9cm 4cm 9cm},clip, scale=0.35]{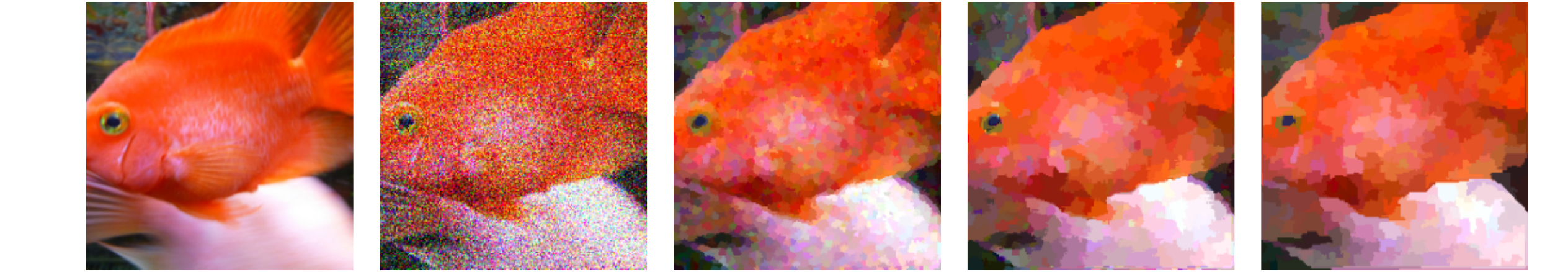}
    \end{subfigure} 

    \caption{  The effect of different super-pixel schemes on the quality of partition-based noisy images from ImageNet}
    \label{fig: Ablation schemes pic}
\end{figure}

\begin{table}[h]
\centering
\renewcommand{\arraystretch}{1.2}
\resizebox{0.99\columnwidth}{!}{%
\begin{tabular}{|c|c|c|c|c|}
\hline
Dataset & Method & Certified Accuracy & Certified F measure & Avg Superpixel Size \\ \hline
\multirow{4}{*}{ImageNet} & SLIC-PPRS & 68.7 & 66.8 & 1230 \\ \cline{2-5} 
 & QuickShift-PPRS & 72.2 & 71.55 & 925 \\ \cline{2-5} 
 & Felzenszwalb-PPRS & 64.3 & 77.93 & 1760 \\ \cline{2-5} 
 & Vanilla Randomized Smoothing & 58.52 & 65.77 & - \\ \hline
\multirow{4}{*}{CIFAR-10} & SLIC-PPRS & 51.9 & 53.4 & 231 \\ \cline{2-5} 
 & QuickShift-PPRS & 53.0 & 55.74 & 340 \\ \cline{2-5} 
 & Felzenszwalb-PPRS & 55.0 & 58.3 & 167 \\ \cline{2-5} 
 & Vanilla Randomized Smoothing & 35.1 & 32.3 & - \\ \hline
\multirow{4}{*}{MNIST} & SLIC-PPRS & 94.9 & 95.1 & 150 \\ \cline{2-5} 
 & QuickShift-PPRS & 98.2 & 98.5 & 77 \\ \cline{2-5} 
 & Felzenszwalb-PPRS & 97.7 & 97.9 & 66 \\ \cline{2-5} 
 & Vanilla Randomized Smoothing & 79.6 & 78.5 & - \\ \hline
\multirow{4}{*}{Fashion MNIST} & SLIC-PPRS & 80.0 & 79.4 & 150 \\ \cline{2-5} 
 & QuickShift-PPRS & 83.2 & 84.0 & 67 \\ \cline{2-5} 
 & Felzenszwalb-PPRS & 83.0 & 84.2 & 68 \\ \cline{2-5} 
 & Vanilla Randomized Smoothing & 73.2 & 74.2 & - \\ \hline
\end{tabular}%
}
\caption{Certified accuracy and F measure achieved by different super-pixel algorithms compared to vanilla randomized smoothing}
\label{tab:superpixel-results}
\end{table}

\begin{figure}
  \centering
    \includegraphics[trim={0cm 0cm 3cm 0cm}, width=0.8\linewidth]
    {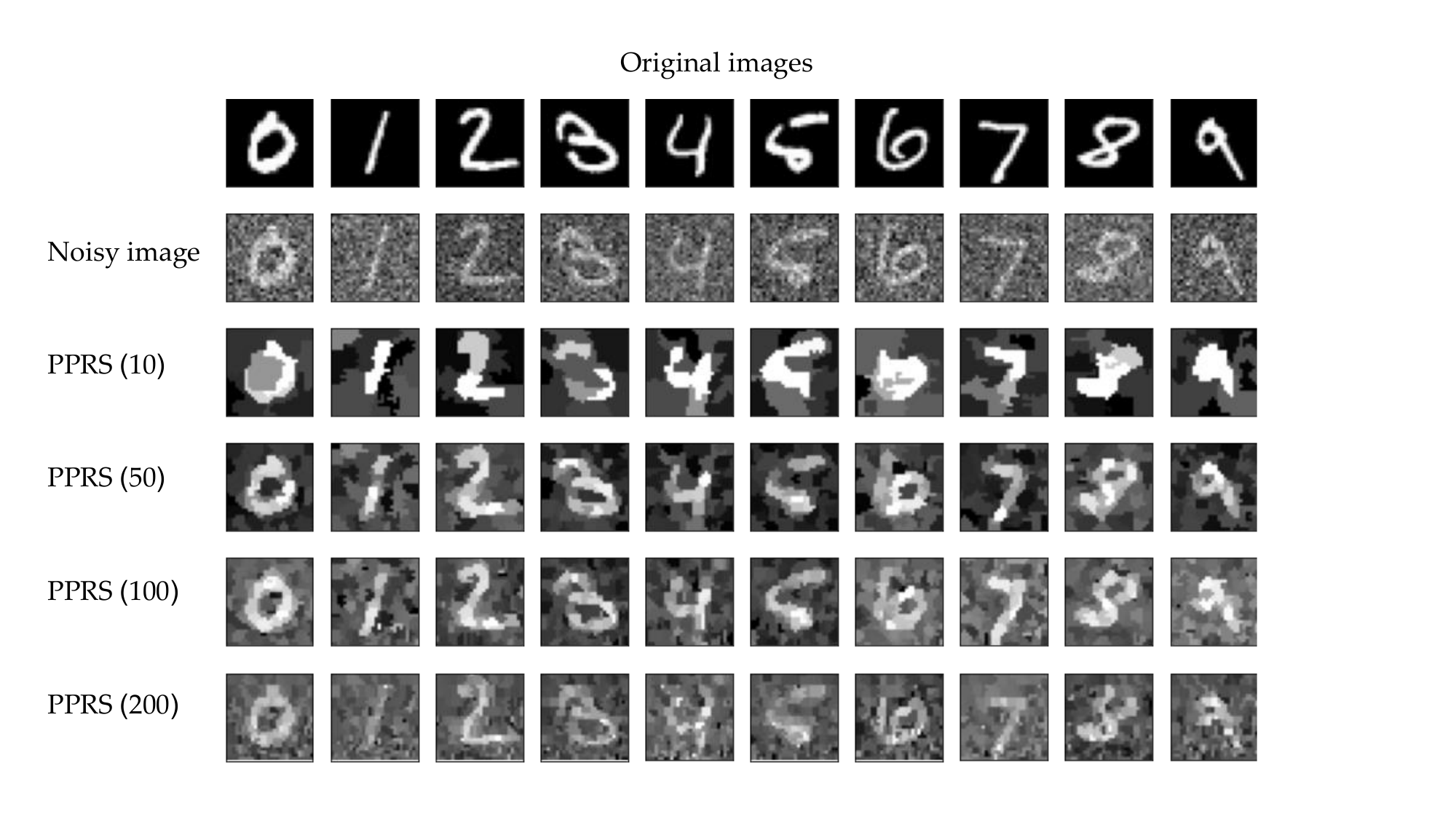}
  \caption{The effect of different numbers of super-pixel components on PPRS partition-based averaged MNIST samples.}
  \label{fig:ablation_mnist}
\end{figure}
\begin{figure}
  \centering
    \includegraphics[trim={0cm 0cm 3cm 0cm}, width=0.8\linewidth]
    {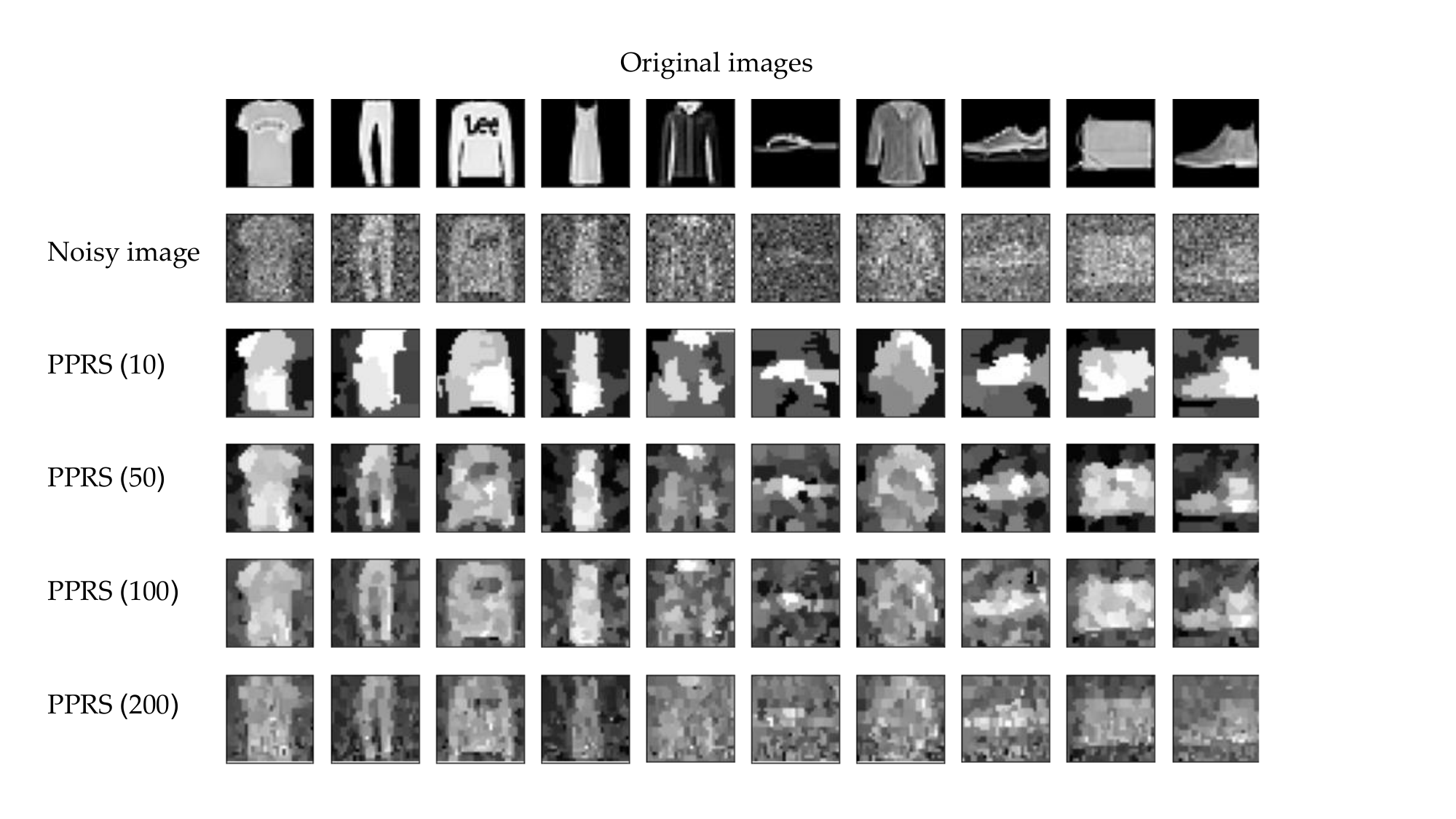}
  \caption{The effect of different numbers of super-pixel components on PPRS partition-based averaged Fashion-MNIST samples.}
  \label{fig:ablation_fashion_mnist}
\end{figure}

\paragraph{Super-pixel Size.}
We tested the effect of the number of super-pixel components on the PPRS results using the SLIC super-pixel method. As expected, when the number of components was selected to be significantly large, super-pixels were shrunk to nearly one pixel, and the super-pixel-based PPRS method was reduced to the vanilla randomized smoothing. On the other hand, when the number of super-pixels was too few, the partitions did not capture the details of the input image, leading to semantically less meaningful images. As shown in Figure \ref{fig: Ablation different sizes}, the number of components should be properly tuned for every dataset. Furthermore, Figures \ref{fig:ablation_mnist} and \ref{fig: Ablation ImgNet sizes pics} show PPRS-recovered versions of noisy MNIST,  Fashion~MNIST, and ImageNet samples using different numbers of super-pixel.


\begin{figure}
    \centering
    \includegraphics[trim={5cm 2.5cm 5cm 2.5cm},clip, scale=0.30]
    {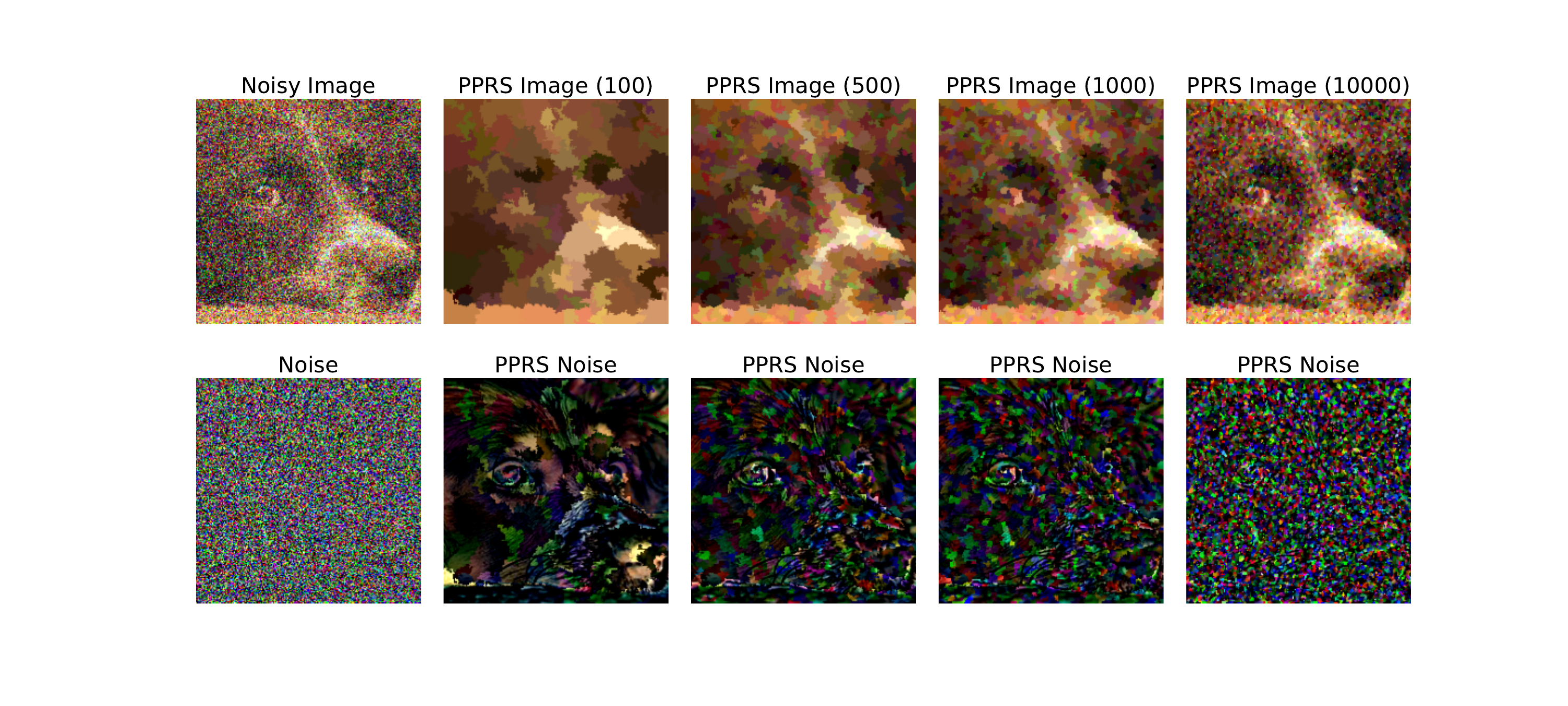}
    \includegraphics[trim={5cm 2.5cm 5cm 2.5cm},clip, scale=0.30]
    {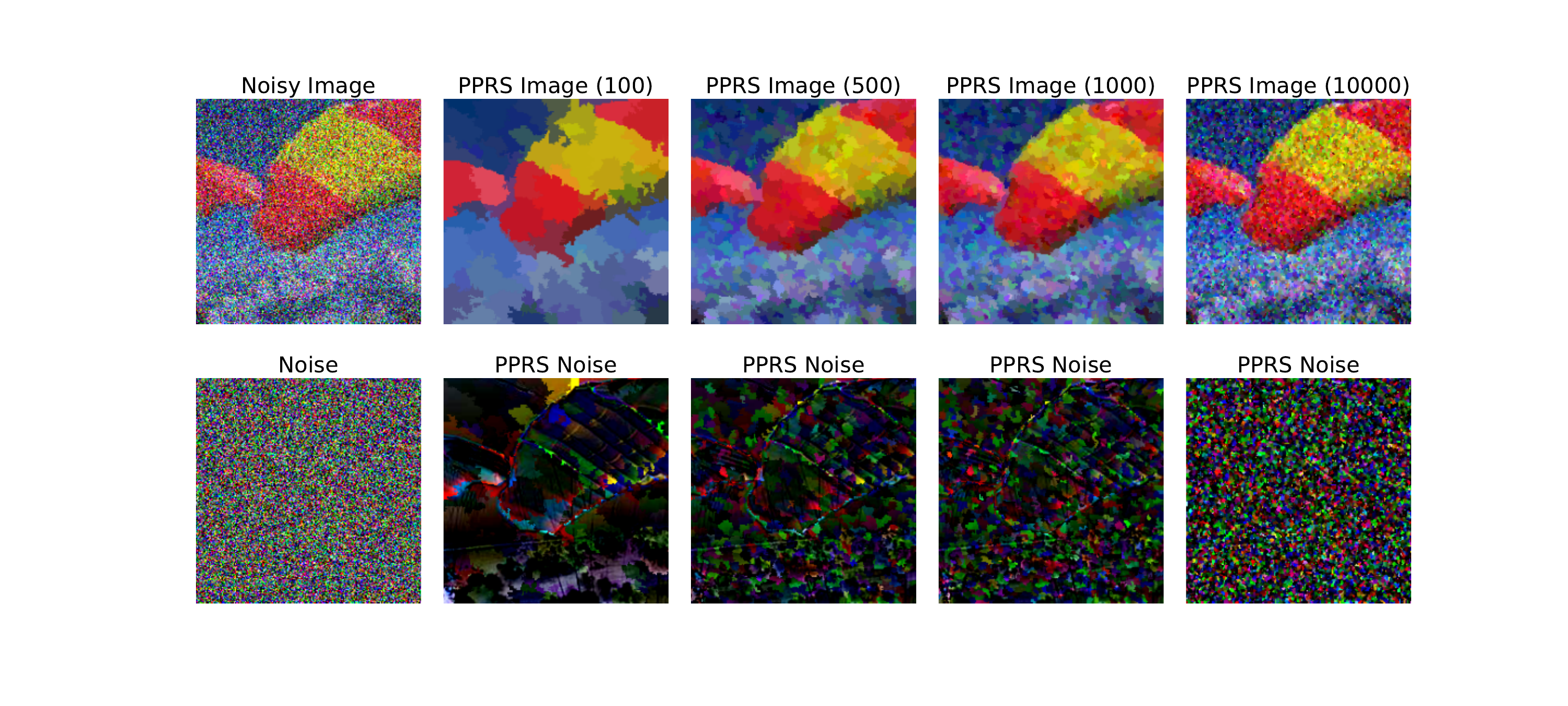}
    \caption{The top and bottom rows demonstrate pictures with additive Gaussian Noise with $\sigma=0.5$ and effect of different super-pixels}
    \label{fig: Ablation ImgNet sizes pics}
\end{figure}


\begin{figure}
    \centering 
\begin{subfigure}{0.33\textwidth}
  \includegraphics[width=\linewidth]{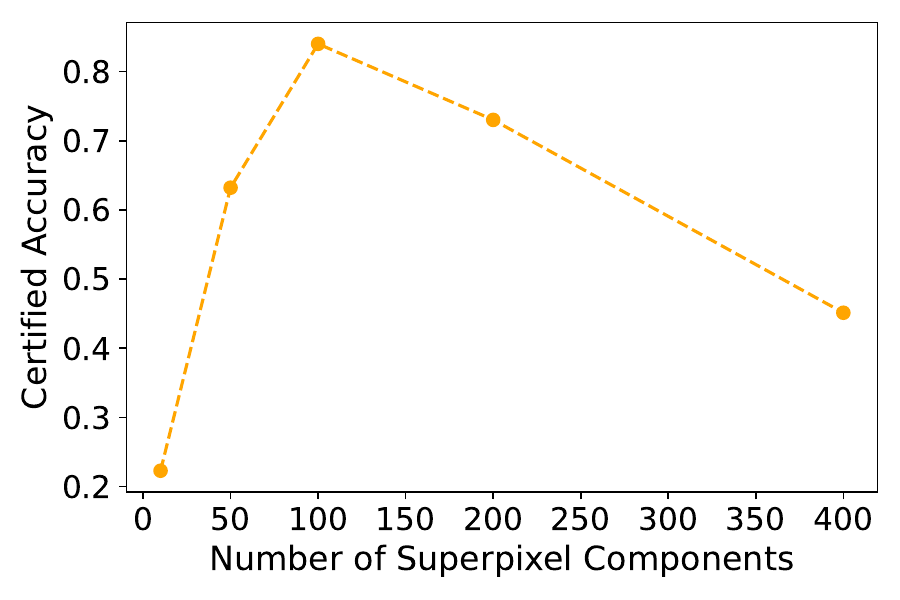}
  \caption{MNIST results}
  \label{fig: Ablation Mnist different sizes}
\end{subfigure}\hfil 
\begin{subfigure}{0.33\textwidth}
  \includegraphics[width=\linewidth]{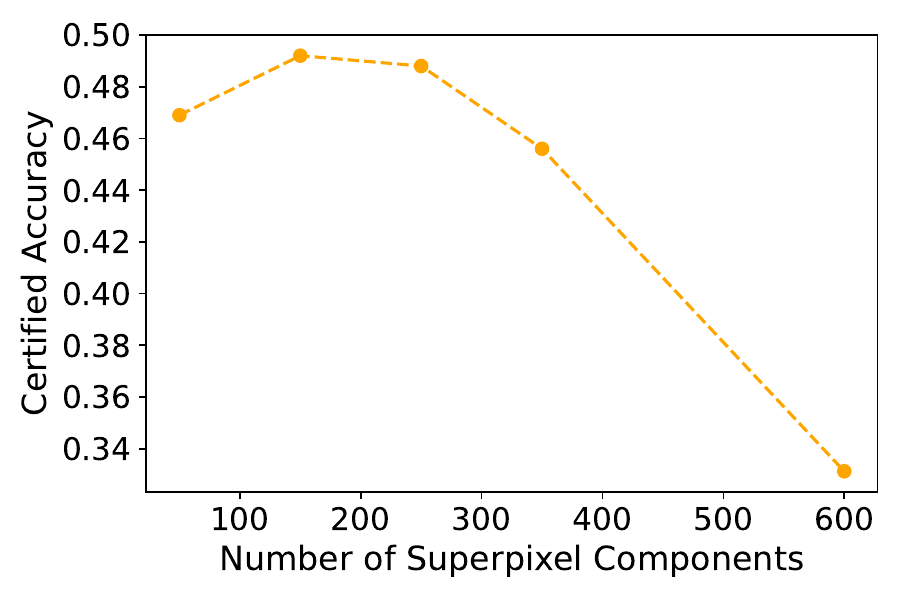}
  \caption{CIFAR-10 results}
  \label{fig: Ablation Cifar different sizes}
\end{subfigure}\hfil 
\begin{subfigure}{0.33\textwidth}
  \includegraphics[width=\linewidth]{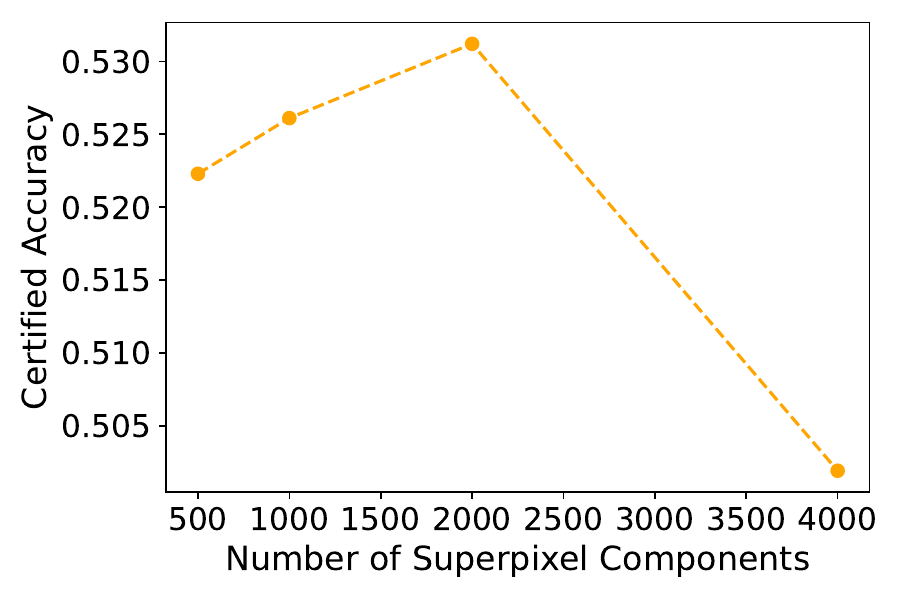}
  \caption{ImageNet results}
  \label{fig: Ablation ImgNet different sizes}
  
\end{subfigure}\hfil 
\caption{The effect of different number of super-pixel components in the certified accuracy obtained for different datasets}
\label{fig: Ablation different sizes}
\end{figure}

\paragraph{PPRS noise hyper-parameters.}

In our experiments, we tested the impact of noise hyperparameters and observed that the proposed super-pixel-based method is more effective under higher noise levels as shown in Figure~\ref{fig: ablation Sigma-Accuracy}. Performing computations with super-pixel variables naturally leads to some bias in the neural network's classification, while the noise reduction will become more effective. Therefore, by including higher levels of noise, the PPRS method performed more effectively until the noise level reached a certain threshold that significantly lowered the classifier's accuracy.


\begin{figure}[h!]
    \centering 
\begin{subfigure}{0.32\textwidth}
  \includegraphics[trim={0cm 0 0cm 0},clip ,width=\linewidth]{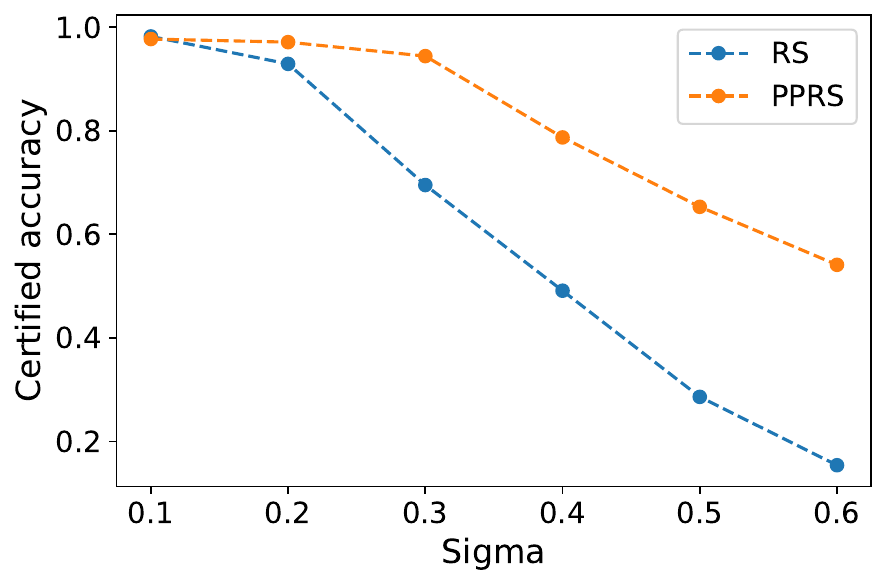}
  \caption{MNIST results}
  \label{fig: Ablation MNIST pic}
\end{subfigure}\hfil 
\begin{subfigure}{0.32\textwidth}
  \includegraphics[trim={0cm 0 0cm 0},clip ,width=\linewidth]{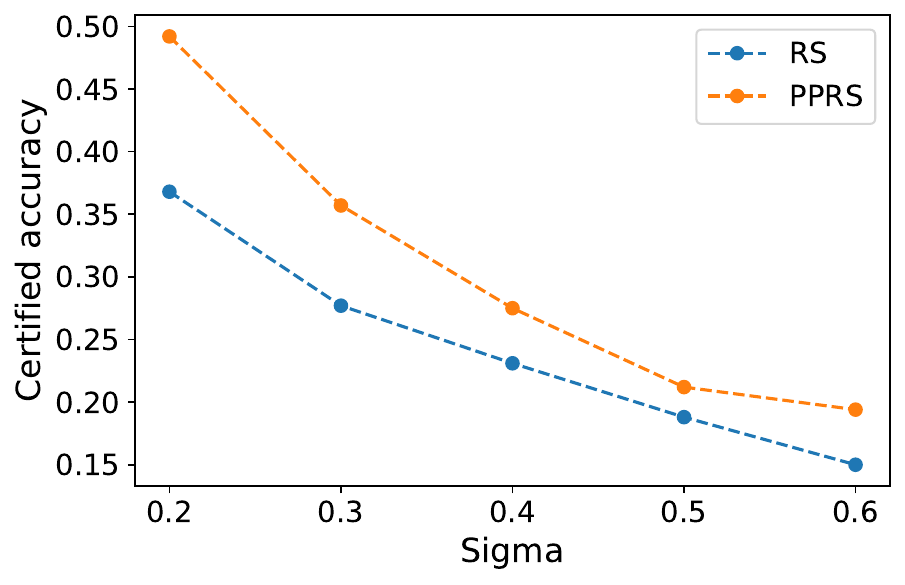}
  \caption{CIFAR-10 results}
  \label{fig: Ablation CIFAR pic}
\end{subfigure}\hfil 
\begin{subfigure}{0.32\textwidth}
  \includegraphics[trim={0cm 0 0cm 0},clip ,width=\linewidth]{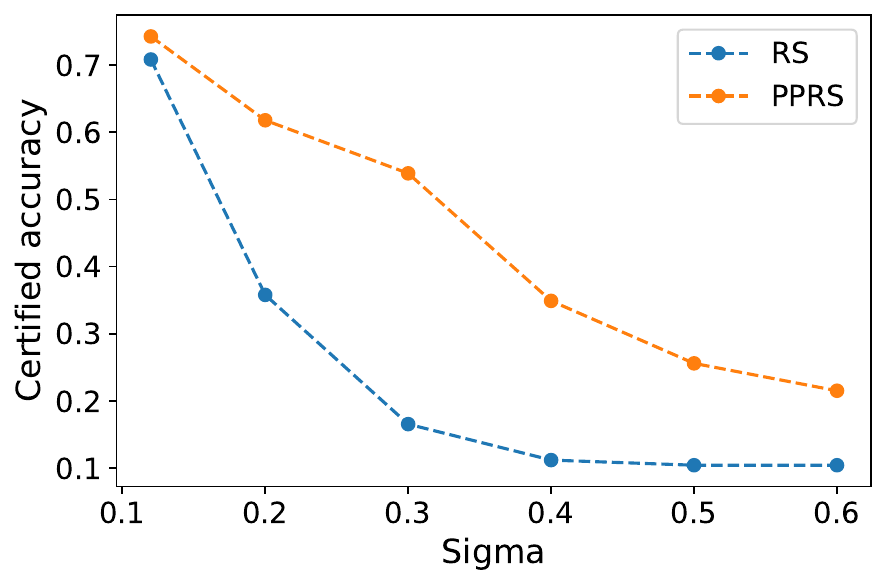}
  \caption{ImageNet results}
  \label{fig: Ablation ImgNet different sizes}
\end{subfigure}\hfil 
  \caption{The effect of different Gaussian noise levels on the resulting certified accuracy}
  \label{fig: ablation Sigma-Accuracy}
\end{figure}

\section{Conclusion}
In this work, we proposed a partitioning-based randomized smoothing algorithm to achieve higher certified robustness scores. Our method is based on grouping the input features into semantically meaningful partitions that can preserve the visibility of the input image. We discussed how such a partitioning scheme could result in a lower input dimension in the application of the Gaussian smoothing method, and proposed applying standard super-pixels to perform the partitioning-based randomized smoothing. Our numerical results suggest the improvements in certified accuracy achieved by the super-pixel-based approach. 
A relevant future direction is to extend the proposed framework to non-computer vision settings, e.g. to text and audio-related classification problems. Moreover, applying the PPRS approach to non-Gaussian randomized smoothing methods will be another interesting topic for future exploration. 

\subsection*{Limitations and Broader Impact}

As discussed in the paper, the proposed super-pixel-based robust classification method applies only to image classification problems where the input image can be clustered into a group of semantically meaningful pixel partitions. Therefore, the extension of the PPRS method to non-image data needs available clustering methods for the input features of the target dataset. In some applications involving complex data types, such a meaningful clustering may not be feasible, e.g. for the genetic datasets in computational biology settings. Therefore, we remark that the proposed method in this work focuses on computer vision settings, while extending the methodology to other domains will remain a future direction to this work.  


{
\bibliography{example_paper}

\begin{thebibliography}{10}

\bibitem{szegedy2013intriguing}
Christian Szegedy, Wojciech Zaremba, Ilya Sutskever, Joan Bruna, Dumitru Erhan, Ian Goodfellow, and Rob Fergus.
\newblock Intriguing properties of neural networks.
\newblock {\em arXiv preprint arXiv:1312.6199}, 2013.

\bibitem{biggio2013evasion}
Battista Biggio, Igino Corona, Davide Maiorca, Blaine Nelson, Nedim {\v{S}}rndi{\'c}, Pavel Laskov, Giorgio Giacinto, and Fabio Roli.
\newblock Evasion attacks against machine learning at test time.
\newblock In {\em Machine Learning and Knowledge Discovery in Databases: European Conference, ECML PKDD 2013, Prague, Czech Republic, September 23-27, 2013, Proceedings, Part III 13}, pages 387--402. Springer, 2013.

\bibitem{goodfellow2015explaining}
Ian~J. Goodfellow, Jonathon Shlens, and Christian Szegedy.
\newblock Explaining and harnessing adversarial examples, 2015.

\bibitem{cohen2019certified}
Jeremy Cohen, Elan Rosenfeld, and Zico Kolter.
\newblock Certified adversarial robustness via randomized smoothing.
\newblock In {\em international conference on machine learning}, pages 1310--1320. PMLR, 2019.

\bibitem{felzenszwalb2006efficient}
Pedro~F Felzenszwalb and Daniel~P Huttenlocher.
\newblock Efficient belief propagation for early vision.
\newblock {\em International journal of computer vision}, 70:41--54, 2006.

\bibitem{vedaldi2008quick}
Andrea Vedaldi and Stefano Soatto.
\newblock Quick shift and kernel methods for mode seeking.
\newblock In {\em Computer Vision--ECCV 2008: 10th European Conference on Computer Vision, Marseille, France, October 12-18, 2008, Proceedings, Part IV 10}, pages 705--718. Springer, 2008.

\bibitem{achanta2012slic}
Radhakrishna Achanta, Appu Shaji, Kevin Smith, Aurelien Lucchi, Pascal Fua, and Sabine S{\"u}sstrunk.
\newblock Slic superpixels compared to state-of-the-art superpixel methods.
\newblock {\em IEEE transactions on pattern analysis and machine intelligence}, 34(11):2274--2282, 2012.

\bibitem{levine2020robustness}
Alexander Levine and Soheil Feizi.
\newblock Robustness certificates for sparse adversarial attacks by randomized ablation.
\newblock In {\em Proceedings of the AAAI Conference on Artificial Intelligence}, volume~34, pages 4585--4593, 2020.

\bibitem{levine2020randomized}
Alexander Levine and Soheil Feizi.
\newblock (de) randomized smoothing for certifiable defense against patch attacks.
\newblock {\em Advances in Neural Information Processing Systems}, 33:6465--6475, 2020.

\bibitem{levine2021improved}
Alexander~J Levine and Soheil Feizi.
\newblock Improved, deterministic smoothing for l\_1 certified robustness.
\newblock In {\em International Conference on Machine Learning}, pages 6254--6264. PMLR, 2021.

\bibitem{zhang2021towards}
Bohang Zhang, Tianle Cai, Zhou Lu, Di~He, and Liwei Wang.
\newblock Towards certifying l-infinity robustness using neural networks with l-inf-dist neurons.
\newblock In {\em International Conference on Machine Learning}, pages 12368--12379. PMLR, 2021.

\bibitem{pmlr-v139-fischer21a}
Marc Fischer, Maximilian Baader, and Martin Vechev.
\newblock Scalable certified segmentation via randomized smoothing.
\newblock In Marina Meila and Tong Zhang, editors, {\em Proceedings of the 38th International Conference on Machine Learning}, volume 139 of {\em Proceedings of Machine Learning Research}, pages 3340--3351. PMLR, 18--24 Jul 2021.

\bibitem{jia2020certified}
Jinyuan Jia, Binghui Wang, Xiaoyu Cao, and Neil~Zhenqiang Gong.
\newblock Certified robustness of community detection against adversarial structural perturbation via randomized smoothing.
\newblock In {\em Proceedings of The Web Conference 2020}, pages 2718--2724, 2020.

\bibitem{pmlr-v119-yang20c}
Greg Yang, Tony Duan, J.~Edward Hu, Hadi Salman, Ilya Razenshteyn, and Jerry Li.
\newblock Randomized smoothing of all shapes and sizes.
\newblock In Hal~Daumé III and Aarti Singh, editors, {\em Proceedings of the 37th International Conference on Machine Learning}, volume 119 of {\em Proceedings of Machine Learning Research}, pages 10693--10705. PMLR, 13--18 Jul 2020.

\bibitem{pmlr-v168-anderson22a}
Brendon~G. Anderson and Somayeh Sojoudi.
\newblock Certified robustness via locally biased randomized smoothing.
\newblock In Roya Firoozi, Negar Mehr, Esen Yel, Rika Antonova, Jeannette Bohg, Mac Schwager, and Mykel Kochenderfer, editors, {\em Proceedings of The 4th Annual Learning for Dynamics and Control Conference}, volume 168 of {\em Proceedings of Machine Learning Research}, pages 207--220. PMLR, 23--24 Jun 2022.

\bibitem{NEURIPS2020_300891a6}
Jeet Mohapatra, Ching-Yun Ko, Tsui-Wei Weng, Pin-Yu Chen, Sijia Liu, and Luca Daniel.
\newblock Higher-order certification for randomized smoothing.
\newblock In H.~Larochelle, M.~Ranzato, R.~Hadsell, M.F. Balcan, and H.~Lin, editors, {\em Advances in Neural Information Processing Systems}, volume~33, pages 4501--4511. Curran Associates, Inc., 2020.

\bibitem{ijcai2022p467}
Nikita Muravev and Aleksandr Petiushko.
\newblock Certified robustness via randomized smoothing over multiplicative parameters of input transformations.
\newblock In Lud~De Raedt, editor, {\em Proceedings of the Thirty-First International Joint Conference on Artificial Intelligence, {IJCAI-22}}, pages 3366--3372. International Joint Conferences on Artificial Intelligence Organization, 7 2022.
\newblock Main Track.

\bibitem{ijcai2023p767}
Taha Belkhouja and Janardhan~Rao Doppa.
\newblock Adversarial framework with certified robustness for time-series domain via statistical features (extended abstract).
\newblock In Edith Elkind, editor, {\em Proceedings of the Thirty-Second International Joint Conference on Artificial Intelligence, {IJCAI-23}}, pages 6845--6850. International Joint Conferences on Artificial Intelligence Organization, 8 2023.
\newblock Journal Track.

\bibitem{maho2022randomized}
Thibault Maho, Teddy Furon, and Erwan Le~Merrer.
\newblock Randomized smoothing under attack: How good is it in practice?
\newblock In {\em ICASSP 2022-2022 IEEE International Conference on Acoustics, Speech and Signal Processing (ICASSP)}, pages 3014--3018. IEEE, 2022.

\bibitem{pmlr-v119-kumar20b}
Aounon Kumar, Alexander Levine, Tom Goldstein, and Soheil Feizi.
\newblock Curse of dimensionality on randomized smoothing for certifiable robustness.
\newblock In Hal~Daumé III and Aarti Singh, editors, {\em Proceedings of the 37th International Conference on Machine Learning}, volume 119 of {\em Proceedings of Machine Learning Research}, pages 5458--5467. PMLR, 13--18 Jul 2020.

\bibitem{9414696}
Muhammad~A. Shah, Raphael Olivier, and Bhiksha Raj.
\newblock Towards adversarial robustness via compact feature representations.
\newblock In {\em ICASSP 2021 - 2021 IEEE International Conference on Acoustics, Speech and Signal Processing (ICASSP)}, pages 3845--3849, 2021.

\bibitem{lecun-mnisthandwrittendigit-2010}
Yann LeCun and Corinna Cortes.
\newblock {MNIST} handwritten digit database.
\newblock 2010.

\bibitem{xiao2017fashion}
Han Xiao, Kashif Rasul, and Roland Vollgraf.
\newblock Fashion-mnist: a novel image dataset for benchmarking machine learning algorithms.
\newblock {\em arXiv preprint arXiv:1708.07747}, 2017.

\bibitem{cifar}
Alex Krizhevsky, Vinod Nair, and Geoffrey Hinton.
\newblock Cifar-10 (canadian institute for advanced research).

\bibitem{deng2009imagenet}
Jia Deng, Wei Dong, Richard Socher, Li-Jia Li, Kai Li, and Li~Fei-Fei.
\newblock Imagenet: A large-scale hierarchical image database.
\newblock In {\em 2009 IEEE conference on computer vision and pattern recognition}, pages 248--255. Ieee, 2009.

\bibitem{he2016deep}
Kaiming He, Xiangyu Zhang, Shaoqing Ren, and Jian Sun.
\newblock Deep residual learning for image recognition.
\newblock In {\em Proceedings of the IEEE conference on computer vision and pattern recognition}, pages 770--778, 2016.

\end{thebibliography}
\bibliographystyle{unsrt}
}

\newpage
\appendix
\onecolumn
\section{Appendix}

\subsection{Proof of Theorem 2}
We begin by showing the following lemma.
\begin{lemma}
    Suppose that matrix $\mathbf{A}_{\mathbb{S}}$ denotes the linear transformation following the group averaging of partitions in $\mathbb{S}(\mathbf{x}):=\{S_1(\mathbf{x}), \cdots, S_p(\mathbf{x})\}$. Here, we define every $(i,j)$th entry of $\mathbf{A}_{\mathbb{S}} \in [0, 1]^{d\times p}$ as $\mathbf{A}_{i,j}=\frac{1}{|S_j|}$ if the $i$-th pixel belongs to $S_j$ with size $|S_j|$ and $\mathbf{A}_{i,j}=0$ otherwise. Then,
    \begin{equation*}
       \mathbf{A}^2_{\mathbb{S}}  = \mathbf{A}_{\mathbb{S}}
    \end{equation*}
\end{lemma}
\begin{proof}
To show this theorem note that for every $i,j$
\begin{align*}
    \mathbf{A}^2_{\mathbb{S}}\Bigr\vert_{i,j}&= \sum_{k=1}^d \mathbf{A}_{\mathbb{S}}\bigr\vert_{i,k}
    \mathbf{A}_{\mathbb{S}}\bigr\vert_{k,j} \\ &= \sum_{k\in S(i) \cap S(j)} \mathbf{A}_{\mathbb{S}}\bigr\vert_{i,k}
    \mathbf{A}_{\mathbb{S}}\bigr\vert_{k,j} \\
    &=\begin{cases}
    0\quad &\text{\rm if}\; S_i\neq S_j \\
    \vert S_i\vert\frac{1}{\vert S_i\vert^2}=\frac{1}{\vert S_i\vert} \quad &\text{\rm if}\; S_i= S_j 
    \end{cases} \\
    &= \mathbf{A}_{\mathbb{S}}\Bigr\vert_{i,j}.
\end{align*}
The above completes the lemma's proof.
\end{proof}
Using the above lemma, we show that the map  $h_{\mathcal{S}}(\mathbf{x}) := A_{\mathcal{S}(\mathbf{x})}\mathbf{x}$ satisfies the following for every vectors $\mathbf{x},\boldsymbol{\delta}$
\begin{align*}
    &\Bigl\Vert h_{\mathcal{S}}(\mathbf{x}+\boldsymbol{\delta}) -  h_{\mathcal{S}}\bigl(\mathbf{x} + A_{\mathcal{S}(\mathbf{x})}\boldsymbol{\delta} \bigr)\Bigr\Vert \\
    =\: &\Bigl\Vert A_{\mathcal{S}(\mathbf{x}+\boldsymbol{\delta})}\bigl(\mathbf{x}+\boldsymbol{\delta}\bigr) -  A_{\mathcal{S}(\mathbf{x} + A_{\mathcal{S}(\mathbf{x})}\boldsymbol{\delta} )}\bigl(\mathbf{x} + A_{\mathcal{S}(\mathbf{x})}\boldsymbol{\delta} \bigr)\Bigr\Vert \\
   =\: &\Bigl\Vert A_{\mathcal{S}(\mathbf{x}+\boldsymbol{\delta})}\bigl(\mathbf{x}+A_{\mathcal{S}(\mathbf{x}+\boldsymbol{\delta})}\boldsymbol{\delta}\bigr) -  A_{\mathcal{S}(\mathbf{x} + A_{\mathcal{S}(\mathbf{x})}\boldsymbol{\delta} )}\bigl(\mathbf{x} + A_{\mathcal{S}(\mathbf{x})}\boldsymbol{\delta} \bigr)\Bigr\Vert \\ 
   \le \: &  \Bigl\Vert A_{\mathcal{S}(\mathbf{x}+\boldsymbol{\delta})}\bigl(\mathbf{x}+A_{\mathcal{S}(\mathbf{x}+\boldsymbol{\delta})}\boldsymbol{\delta}\bigr) -  A_{\mathcal{S}(\mathbf{x} + A_{\mathcal{S}(\mathbf{x})}\boldsymbol{\delta} )}\bigl(\mathbf{x} + A_{\mathcal{S}(\mathbf{x}+\boldsymbol{\delta})}\boldsymbol{\delta} \bigr)\Bigr\Vert   \\
   &\quad +       \Bigl\Vert A_{\mathcal{S}(\mathbf{x} + A_{\mathcal{S}(\mathbf{x})}\boldsymbol{\delta} )}\bigl(\mathbf{x} + A_{\mathcal{S}(\mathbf{x}+\boldsymbol{\delta})}\boldsymbol{\delta} \bigr) -  A_{\mathcal{S}(\mathbf{x} + A_{\mathcal{S}(\mathbf{x})}\boldsymbol{\delta} )}\bigl(\mathbf{x} + A_{\mathcal{S}(\mathbf{x})}\boldsymbol{\delta} \bigr)\Bigr\Vert \\
   = \: &  \Bigl\Vert \Bigl(A_{\mathcal{S}(\mathbf{x}+\boldsymbol{\delta})} - A_{\mathcal{S}(\mathbf{x} + A_{\mathcal{S}(\mathbf{x})}\boldsymbol{\delta} )} \Bigr)\bigl(\mathbf{x}+A_{\mathcal{S}(\mathbf{x}+\boldsymbol{\delta})}\boldsymbol{\delta}\bigr) \Bigr\Vert     +       \Bigl\Vert A_{\mathcal{S}(\mathbf{x} + A_{\mathcal{S}(\mathbf{x})}\boldsymbol{\delta} )}\bigl( A_{\mathcal{S}(\mathbf{x}+\boldsymbol{\delta})} - A_{\mathcal{S}(\mathbf{x})}\bigr)\boldsymbol{\delta}    \Bigr\Vert \\
   \le \: &  \rho \Vert I - A_{\mathcal{S}(\mathbf{x})} \Vert \Vert \boldsymbol{\delta}\Vert \bigl( \Vert \mathbf{x}\Vert + \Vert \boldsymbol{\delta}\Vert\bigr) + \rho\Vert A_{\mathcal{S}(\mathbf{x} + A_{\mathcal{S}(\mathbf{x})}\boldsymbol{\delta} )}\Vert  \Vert \boldsymbol{\delta}\Vert^2  \\
   \le \:& 2\rho\Vert\boldsymbol{\delta}\Vert^2 + \rho\Vert\boldsymbol{\delta}\Vert\Vert\mathbf{x}\Vert.
\end{align*}
Therefore, given every $\mathbf{x}$, we consider the certified robustness guarantee of Corollary 1 for the fixed partitioning matrix $A_{\mathcal{S}(\mathbf{x})}$, which shows the labeling assigned by the Gaussian smoothed version of $h_{S(\mathbf{x})}(\mathbf{x}+A_{\mathcal{S}(\mathbf{x})}\boldsymbol{\delta})$ remains unchanged if $\Vert A_{\mathcal{S}(\mathbf{x})}\boldsymbol{\delta}\Vert \le \sigma C_{\mathrm{PPRS}(f^{\mathrm{GS}(\sigma)})}(\mathbf{x})$. On the other hand, Theorem 1 from \cite{cohen2019certified} together with the above bound imply that under a perturbation $\boldsymbol{\delta}$ satisfying $2\rho\Vert\boldsymbol{\delta}\Vert^2 + \rho\Vert\boldsymbol{\delta}\Vert\Vert\mathbf{x}\Vert \le \sigma C_{\mathrm{PPRS}(f^{\mathrm{GS}(\sigma)})}(\mathbf{x})$ the labels assigned by $\mathrm{f}^{PPRS}$ to $\mathbf{x}+\boldsymbol{\delta}$ and $\mathbf{x}+A_S(\mathbf{x})\boldsymbol{\delta}$ will be the same. Therefore, assuming that $\Vert\boldsymbol{\delta}\Vert + 2\rho\Vert\boldsymbol{\delta}\Vert^2 + \rho\Vert\boldsymbol{\delta}\Vert\Vert\mathbf{x}\Vert\le \sigma C_{\mathrm{PPRS}(f^{\mathrm{GS}(\sigma)})}(\mathbf{x})$, the certified robustness will hold, which under the assumption that $\max\{\Vert\mathbf{x}\Vert, \Vert\boldsymbol{\delta}\Vert\}\le 1$ will imply that if $(1+3\rho)\Vert\boldsymbol{\delta}\Vert \le \sigma C_{\mathrm{PPRS}(f^{\mathrm{GS}(\sigma)})}(\mathbf{x})$, the certified robustness will hold and the prediction for $\mathbf{x}$ and $\mathbf{x}+\boldsymbol{\delta}$ will be identical.

\subsection{Additional Numerical Results and Visualizatons}
To further investigate our method, we will also compare it with Yang et~al. \cite{pmlr-v119-yang20c} proposed framework for certifying a robust radius for a given classifier. They introduce the notion of Wulff Crystal and spherical level sets and propose a general framework that can use different distribution noises and find a correspondent robust radii. Because our work focuses on $L_2$~norm we choose Exponential, PowerLaw, and Gaussian distributions and follow Yang et~al.  method. We compare these baselines using different $\sigma$'s in Fig.~\ref{AppendixAllDistImgNet} and Fig.~\ref{AppendixAllDistMnist}. 

\begin{itemize}
    \item Exponential $\propto e^-{||\frac{x}{\lambda}||_2}$ 
    \item Gaussian $\propto e^{-||\frac{x}{\lambda}||_2^2}$ 
    \item Power Law $\propto (1+||\frac{x}{\lambda}||_2)^{-a}$ 
\end{itemize}

Fig.~\ref{fig: Appendix SampleImgs} demonstrates more samples and their respective Certified Radius when using PPRS vs RS (100 samples for estimating the class and 1000 samples for estimating the radius using \eqref{Thm: Thm 1 Cohen}). Also here we chose Resnet 101 as the base classifier. As shown in the figures, we believe the Superpixel Algorithm is more effective in Images with simpler geometry where similar clusters of pixels are more apparent.

\begin{figure*}[h!]
    \centering 
\begin{subfigure}{0.485\textwidth}
  \includegraphics[trim={0 0 0 0},clip ,width=\linewidth]{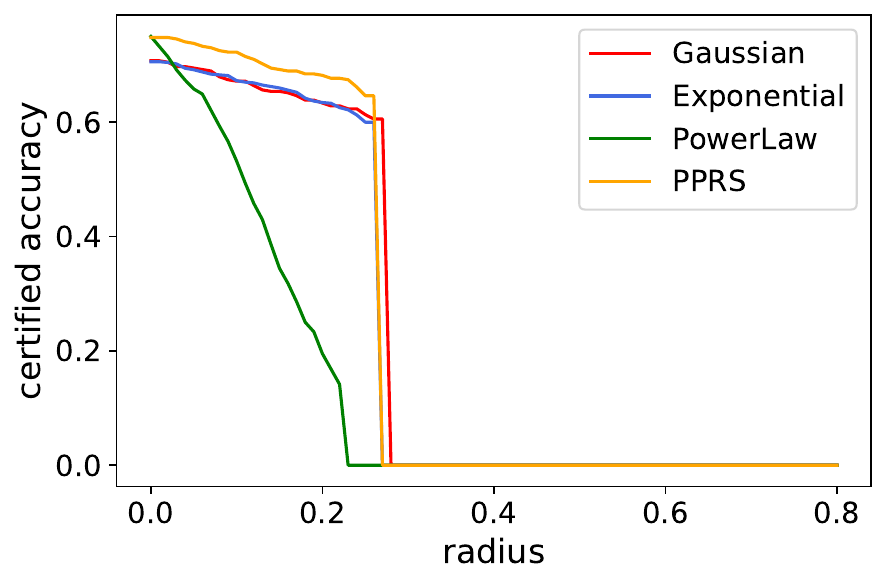}
    \caption{ImageNet $\sigma=0.12$}
  \label{fig:AppendexDists12Img}
\end{subfigure}\hfil 
\begin{subfigure}{0.485\textwidth}
  \includegraphics[trim={0 0 0 0},clip ,width=\linewidth]{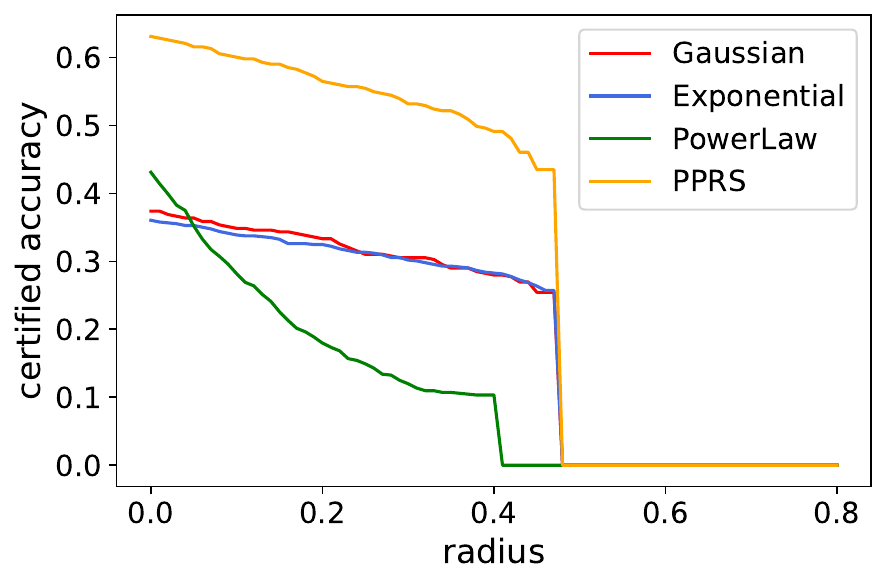}
    \caption{ImageNet $\sigma=0.2$}
  \label{fig:AppendexDists20Img}
\end{subfigure}\hfil 
\begin{subfigure}{0.485\textwidth}
  \includegraphics[trim={0 0 0 0},clip ,width=\linewidth]{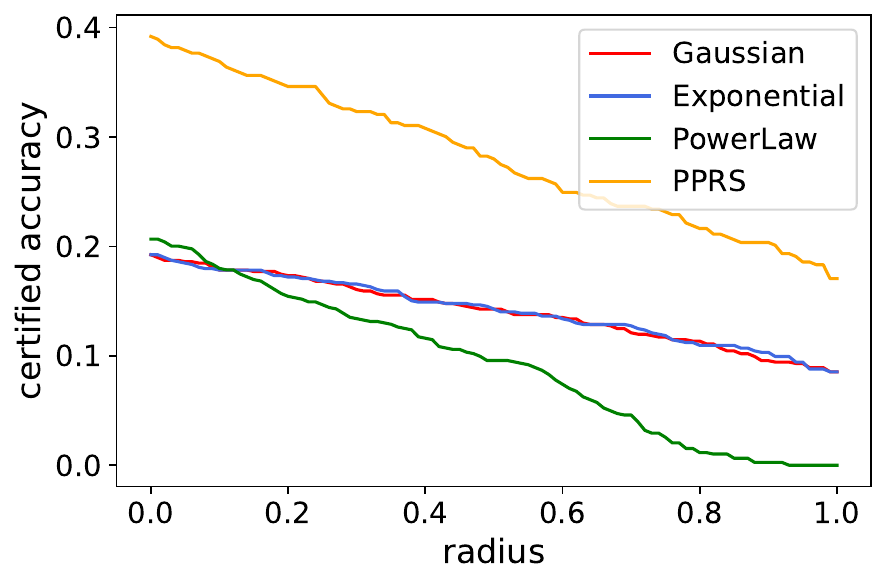}
  \caption{ImageNet $\sigma=0.5$}
  \label{fig:AppendexDists50Img}
\end{subfigure} 

\caption{Comparision of our method with Yang et~al. framework for randomized smoothing when using a $L_2$ adversary on ImageNet}
\label{AppendixAllDistImgNet}
\end{figure*}

\begin{figure*}[h]
    \centering 
\begin{subfigure}{0.485\textwidth}
  \includegraphics[trim={0 0 0 0},clip ,width=\linewidth]{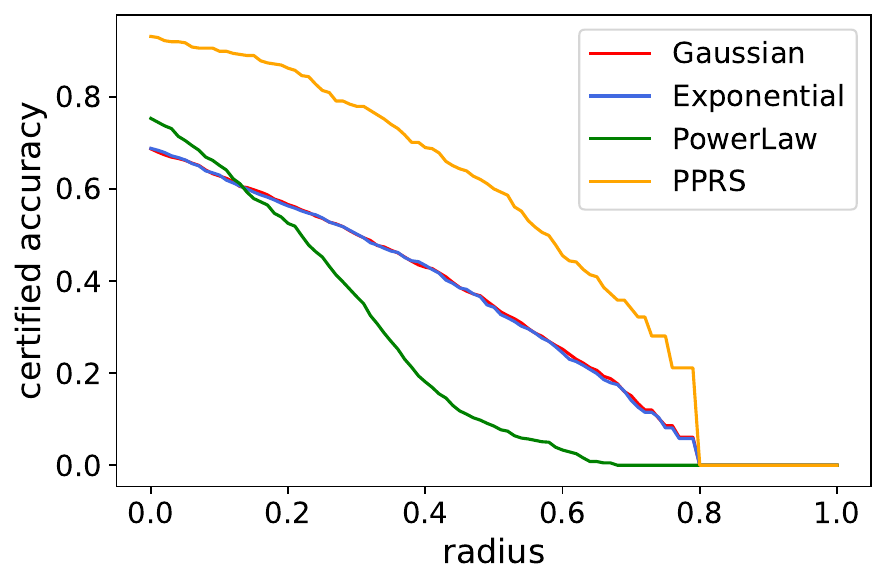}
    \caption{MNIST $\sigma=0.3$}
  \label{fig:AppendexDists30Mnist}
\end{subfigure}\hfil 
\begin{subfigure}{0.485\textwidth}
  \includegraphics[trim={0 0 0 0},clip ,width=\linewidth]{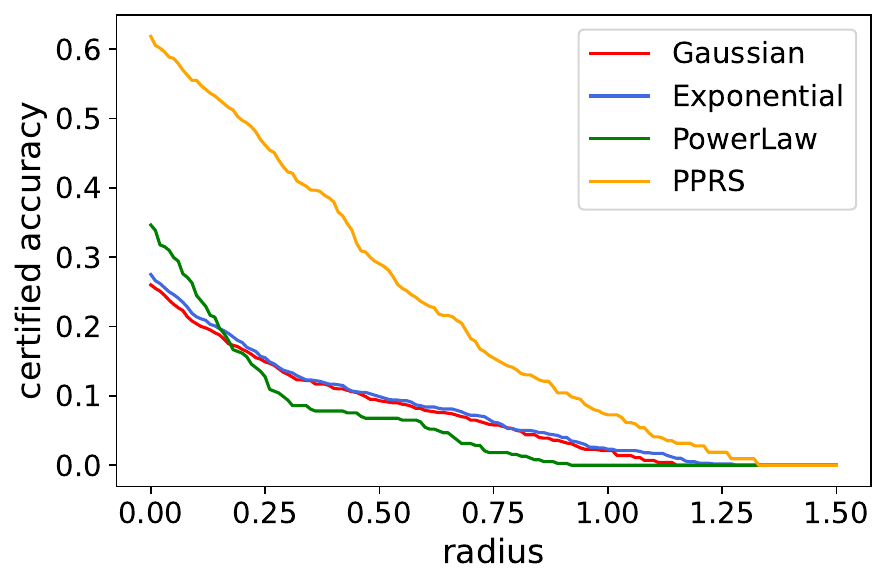}
    \caption{MNIST $\sigma=0.5$}
  \label{fig:AppendexDists50Mnist}
\end{subfigure}\hfil 
\begin{subfigure}{0.485\textwidth}
  \includegraphics[trim={0 0 0 0},clip ,width=\linewidth]{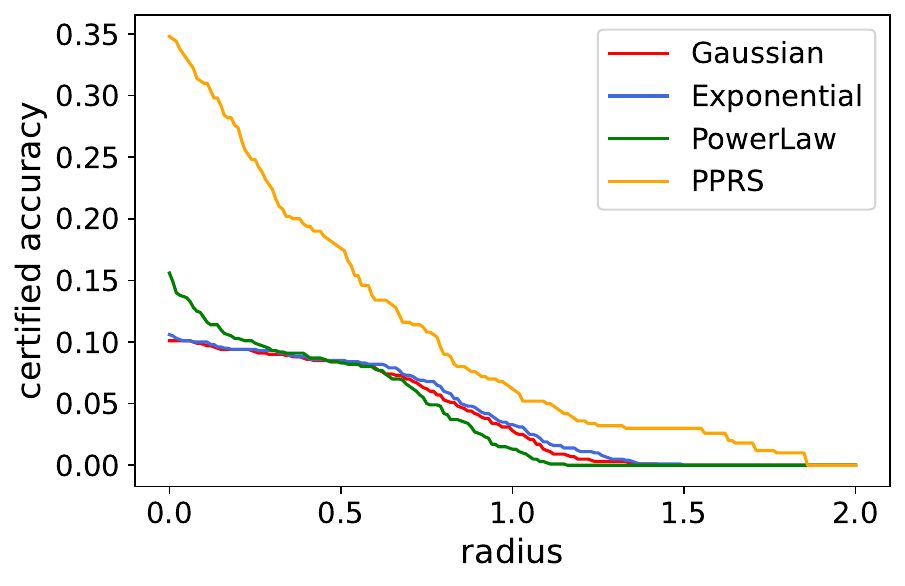}
  \caption{MNIST $\sigma=0.7$}
  \label{fig:AppendexDists70Mnist}
\end{subfigure}
\begin{subfigure}{0.485\textwidth}
  \includegraphics[trim={0 0 0 0},clip ,width=\linewidth]{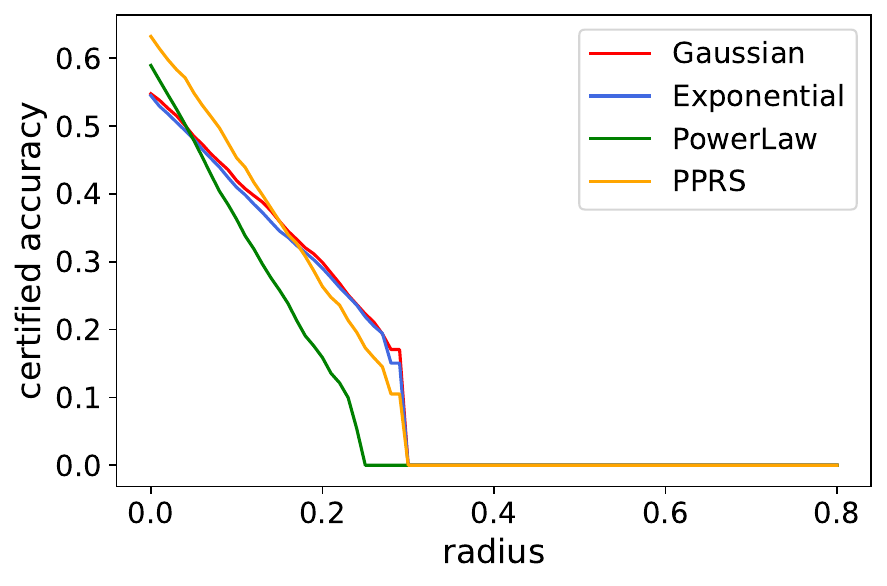}
    \caption{CIFAR-10 $\sigma=0.12$}
  \label{fig:AppendexDists12Cifar}
\end{subfigure}\hfil 
\begin{subfigure}{0.485\textwidth}
  \includegraphics[trim={0 0 0 0},clip ,width=\linewidth]{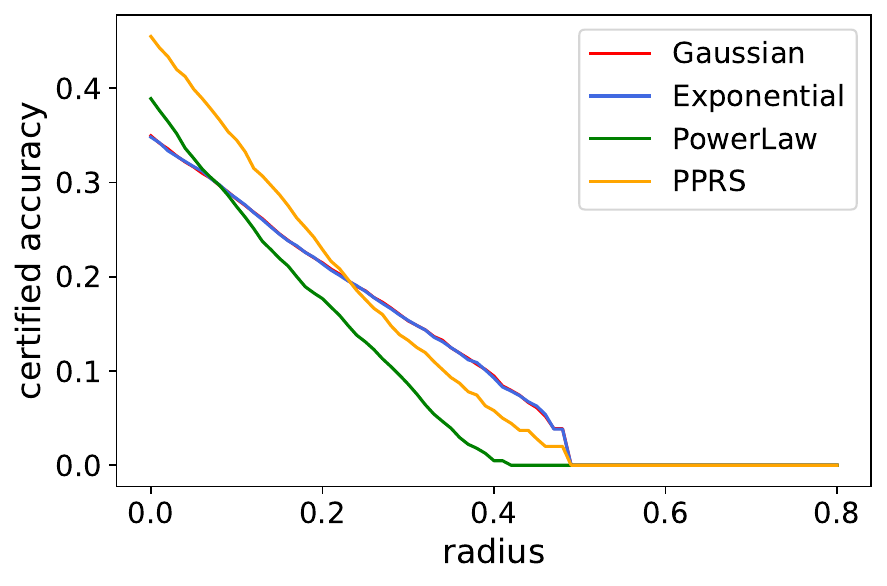}
    \caption{CIFAR-10 $\sigma=0.2$}
  \label{fig:AppendexDists20Cifar}
\end{subfigure}\hfil 
\begin{subfigure}{0.485\textwidth}
  \includegraphics[trim={0 0 0 0},clip ,width=\linewidth]{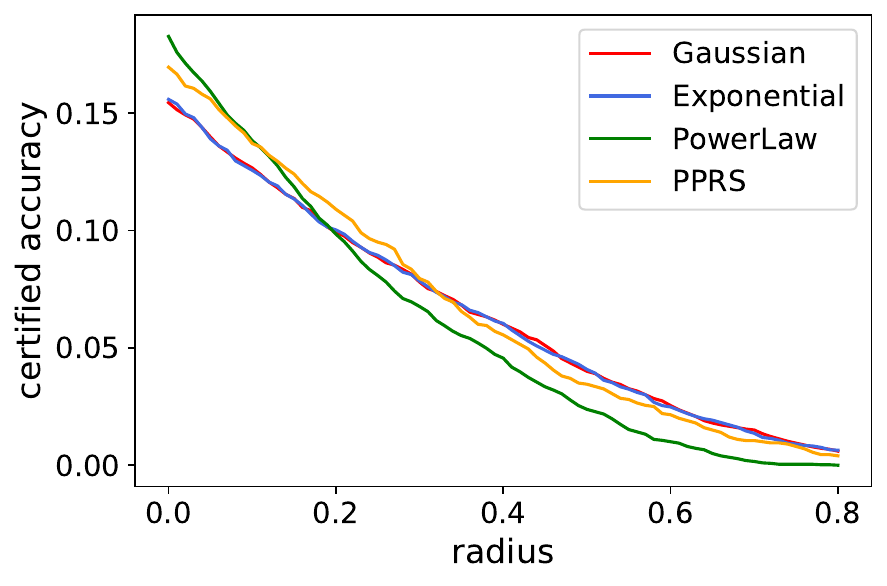}
  \caption{CIFAR-10 $\sigma=0.5$}
  \label{fig:AppendexDists50Cifar}
\end{subfigure}\hfil 
\caption{Comparision of our method with Yang et~al. framework for randomized smoothing when using a $L_2$ adversary on MNIST and CIFAR-10}
\label{AppendixAllDistMnist}
\end{figure*}

\begin{figure*}[h]
    \centering
    \begin{subfigure}{\textwidth}
    \includegraphics[trim={5cm 19cm 5cm 21cm},clip, scale=0.4]{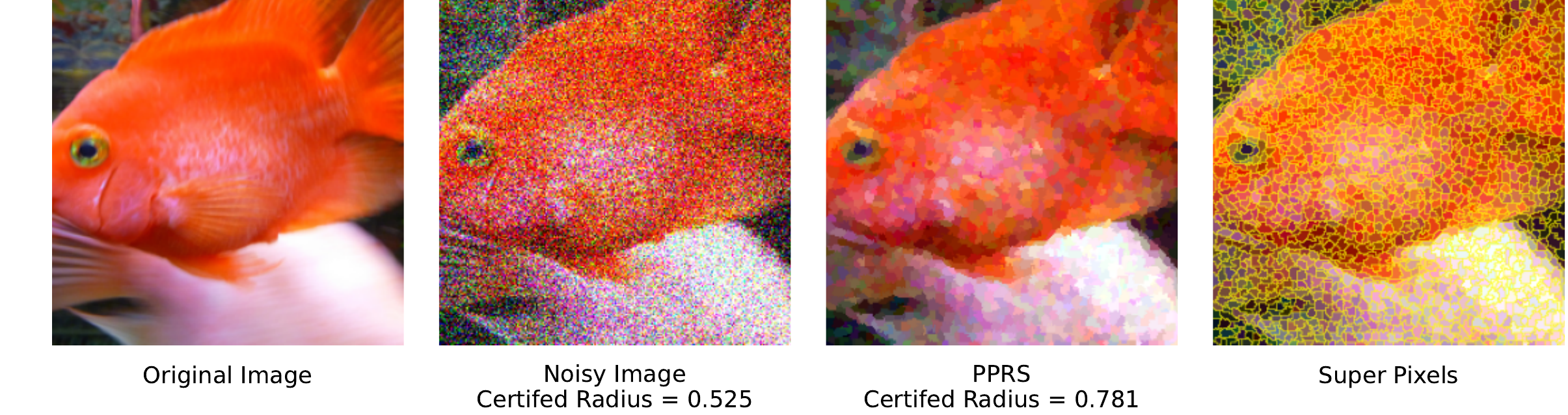}
    \label{fig:Appendix SampleImgs subfig1}
    \end{subfigure} 
    \begin{subfigure}{\textwidth}
    \includegraphics[trim={5cm 19cm 5cm  21cm},clip, scale=0.4]{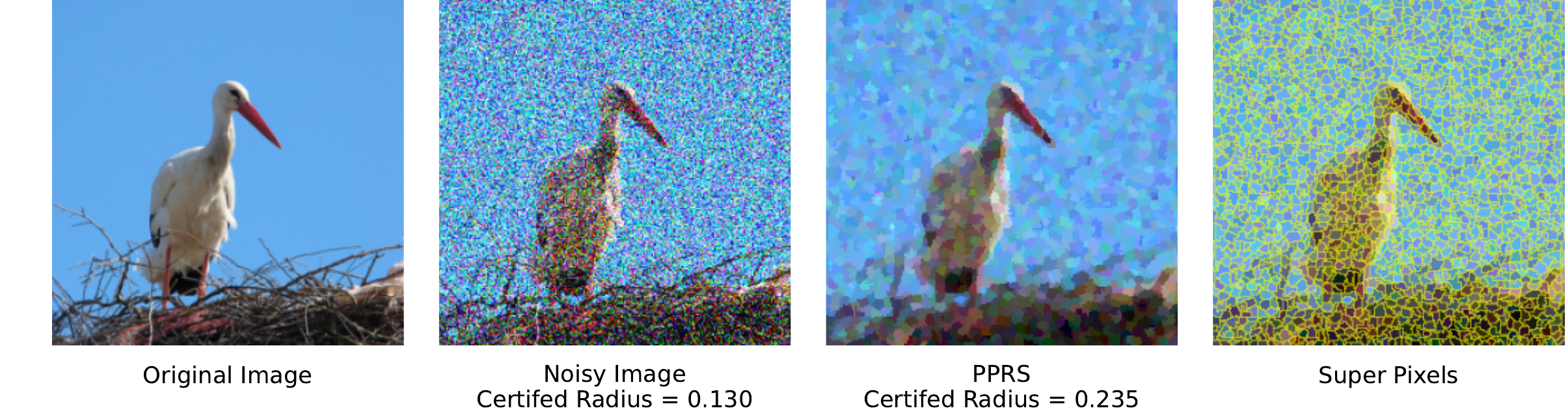}
    \label{fig:Appendix SampleImgs subfig2}
    \end{subfigure}\hfil 
    \begin{subfigure}{\textwidth}
    \includegraphics[trim={5cm 19cm 5cm 21cm},clip, scale=0.4]{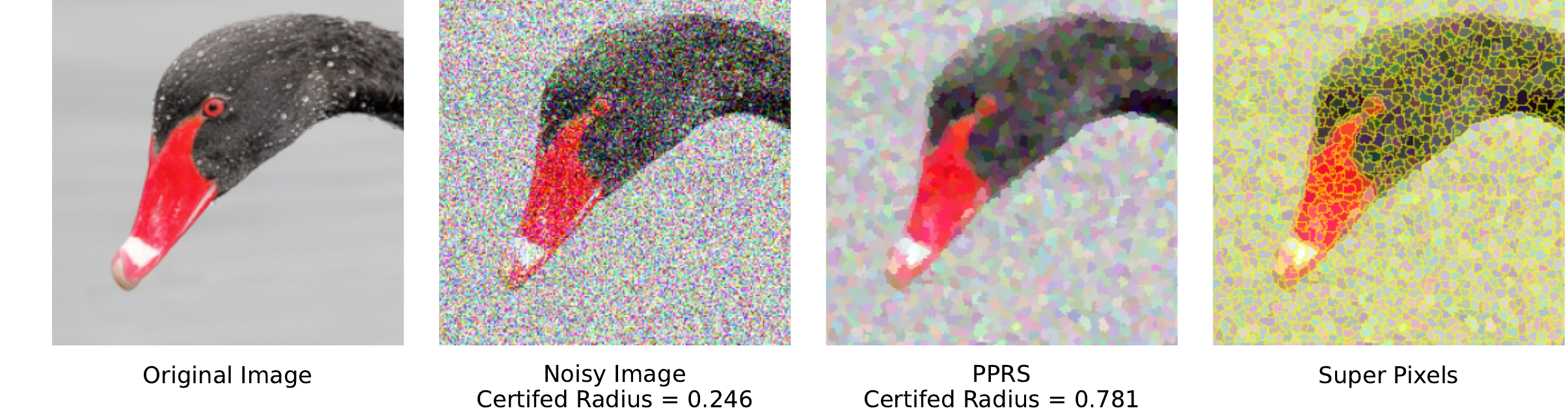}
    \label{fig:Appendix SampleImgs subfig3}
    \end{subfigure}\hfil 
    \begin{subfigure}{\textwidth}
    \includegraphics[trim={5cm 19cm 5cm 21cm},clip, scale=0.4]{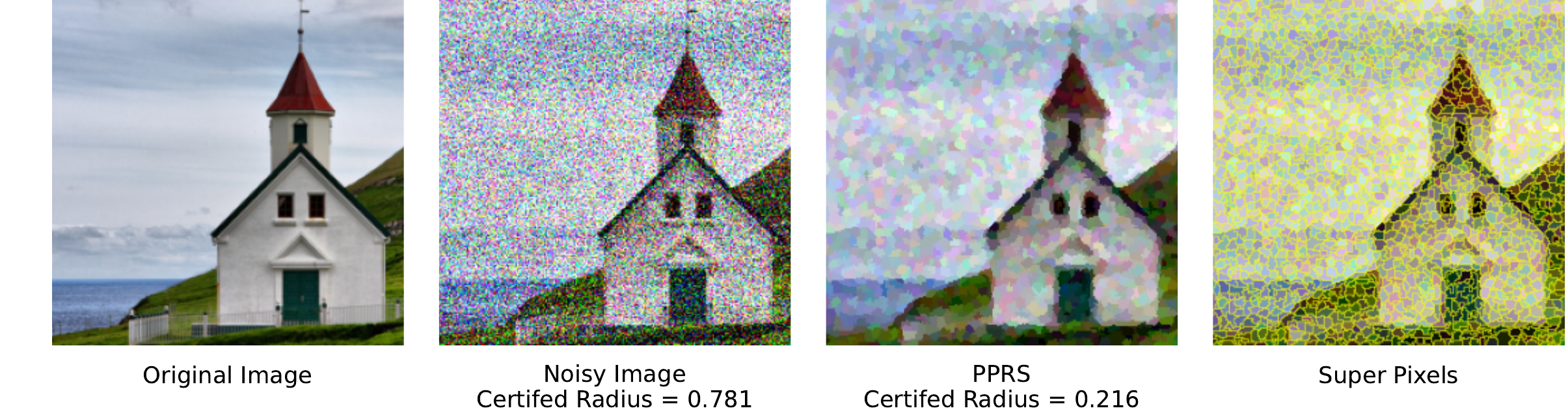}
    \label{fig:Appendix SampleImgs subfig4}
    \end{subfigure}\hfil 
    \begin{subfigure}{\textwidth}
    \includegraphics[trim={5cm 19cm 5cm 21cm},clip, scale=0.4]{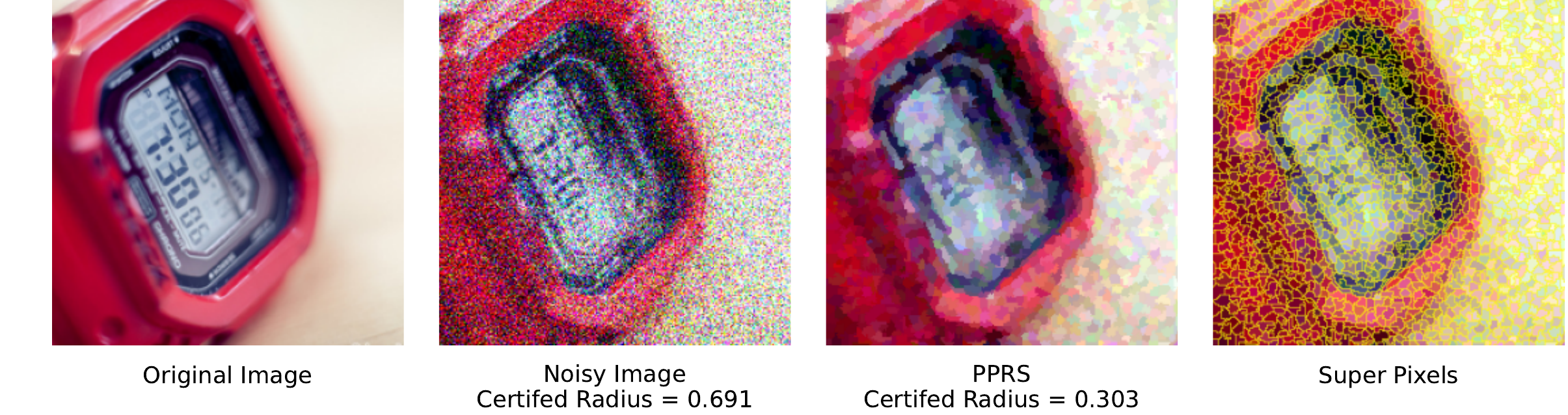}
    \label{fig:Appendix SampleImgs subfig4]5}
    \end{subfigure}\hfil 
    \caption{  Additional demonstration images found after adding Gaussian Noise with $\sigma$ = 0.3 and
applying SuperPixel with 2000 components.}
    \label{fig: Appendix SampleImgs}
\end{figure*}


\end{document}